\documentclass[twoside,11pt]{article}

\usepackage{blindtext}

\usepackage[abbrvbib, preprint]{jmlr2e}
\usepackage{amsmath}
\usepackage{bm}
\usepackage{mathtools}
\usepackage{bbm}
\usepackage{mathrsfs}
\usepackage{booktabs}
\usepackage{array}
\usepackage{enumitem}
\usepackage{xcolor}

\usepackage{tikz}
\usetikzlibrary{shapes,arrows,positioning,calc,decorations.pathmorphing}
\newtheorem{assumption}[theorem]{Assumption}

\usepackage{lastpage}

\ShortHeadings{A Spectral Revisit of the Distributional Bellman Operator under the Cram\'er Metric}{Wang, Deng, Lyu, Redmond and Li}
\firstpageno{1}

\begin{document}

\title{A Spectral Revisit of the Distributional Bellman Operator under the Cram\'er Metric}

\author{\name Keru Wang\textsuperscript{1} \email keru.wang@ucdconnect.ie \\
       \name Yixin Deng\textsuperscript{1} \email yixin.deng@ucdconnect.ie \\
       \name Yao Lyu\textsuperscript{2} \email lvyao@mail.tsinghua.edu.cn \\
       \name Stephen Redmond\textsuperscript{1}\thanks{Corresponding authors.} \email stephen.redmond@ucd.ie \\
       \name Shengbo Eben Li\textsuperscript{2,3}\footnotemark[1] \email lishbo@tsinghua.edu.cn \\
       \addr \textsuperscript{1}School of Electrical and Electronic Engineering, University College Dublin, Dublin, Ireland\\
       \addr \textsuperscript{2}School of Vehicle and Mobility, Tsinghua University, Beijing, China\\
       \addr \textsuperscript{3}College of Artificial Intelligence, Tsinghua University, Beijing, China}

\editor{My editor}

\maketitle

\begin{abstract}
Distributional reinforcement learning (DRL) studies the evolution of full return distributions under Bellman updates rather than focusing on expected values. A classical result is that the distributional Bellman operator is contractive under the Cram\'er metric, which corresponds to an $L^2$ geometry on differences of cumulative distribution functions (CDFs). While this contraction ensures stability of policy evaluation, existing analyses remain largely metric, focusing on contraction properties without elucidating the structural action of the Bellman update on distributions. In this work, we analyse distributional Bellman dynamics directly at the level of CDFs, treating the Cram\'er geometry as the intrinsic analytical setting. At this level, the Bellman update acts affinely on CDFs and linearly on differences between CDFs, and its contraction property yields a uniform bound on this linear action. Building on this intrinsic formulation, we construct a family of regularised spectral Hilbert representations that realise the CDF-level geometry by exact conjugation, without modifying the underlying Bellman dynamics. The regularisation affects only the geometry and vanishes in the zero-regularisation limit, recovering the native Cram\'er metric. This framework clarifies the operator structure underlying distributional Bellman updates and provides a foundation for further functional and operator-theoretic analyses in DRL.
\end{abstract}


\section{Introduction}
Distributional reinforcement learning (DRL) models the full law of returns associated with each state-action pair, providing a richer description of uncertainty than expectation-based formulations~\citep{c51,drlbook_bellemare_distributional_2023}. A foundational theoretical result underpinning this paradigm is that the distributional Bellman operator~\citep{drlbook_bellemare_distributional_2023} is contractive~\citep{c51theory} under the Cram\'er metric~\citep{cramer1,cramer2}, which coincides with the $\ell_2$ distance between cumulative distribution functions (CDFs)~\citep{prob_intro}. This contraction property guarantees stability, well-posedness, and convergence of distributional policy evaluation, and forms the theoretical basis of CDF-based DRL methods~\citep{c51,c51theory}.

Despite this well-established metric theory, the internal structure of distributional Bellman dynamics remains poorly understood. Existing analyses of Cram\'er-based DRL are largely confined to contraction arguments~\citep{c51theory}, which are sufficient for establishing existence and uniqueness of fixed points, but offer limited insight into how the Bellman update acts as a transformation on distributions. In particular, they do not clarify how the update (which combines reward translation, discount-induced scaling, and conditional expectation~\citep{Suttonbook}) organises its probabilistic and algebraic components, nor how these components interact under the Cram\'er geometry. As a result, structural, spectral, and operator-level questions about distributional Bellman dynamics remain largely inaccessible.

This gap is not merely technical, but structural. At the level of distributions, the Bellman update is inherently affine rather than linear: it preserves probabilistic structure but does not act linearly on the space of CDFs. The Cram\'er metric, in turn, is intrinsically a geometry of differences between distributions, endowing the CDF space with an $L^2$-type structure only through comparison~\citep{metric1,metric2}, while remaining incompatible with a direct Hilbert space formulation on distributions themselves~\citep{functional1,functional2}, as will be exposed later. Consequently, while contraction and fixed-point properties can be analysed directly at the raw CDF level, the lack of a compatible linear representation obscures the deeper functional organisation of the dynamics and precludes systematic operator-theoretic analysis.

This observation reveals a fundamental representational tension in distributional reinforcement learning: the Bellman dynamics are probabilistic and affine in nature, whereas the analytical tools required to study their structure, such as spectral analysis or perturbation theory, are inherently linear. Resolving this tension requires a framework that respects the intrinsic CDF-level dynamics while making the latent structure encoded by the Cram\'er geometry analytically accessible, without altering the underlying Bellman operator.

We therefore ask: \emph{How can distributional Bellman dynamics be analysed at their natural CDF level while admitting a faithful Hilbert space realisation that preserves the intrinsic Cram\'er geometry and exposes the internal structure of the Bellman update?} Crucially, such a realisation should not modify or approximate the Bellman operator, but instead provide an exact correspondence that separates probabilistic dynamics from analytical representation and recovers the Cram\'er geometry in an appropriate limit.

To address this question, we adopt a two-level analytical strategy. We first analyse distributional Bellman dynamics directly in the CDF domain, treating the Cram\'er geometry as the intrinsic analytical setting and establishing contraction, invariance, and fixed-point properties~\citep{Suttonbook,drlbook_bellemare_distributional_2023,rlbook} without reference to any auxiliary representation. Building on this intrinsic formulation, we then construct a family of regularised Hilbert spaces that provide an exact conjugate realisation of the CDF-level geometry. This construction does not introduce new dynamics; rather, it makes explicit the affine structure of the Bellman update at the CDF level and the geometric structure implicitly encoded by the Cram\'er metric, yielding a functional setting in which operator-theoretic and spectral tools become applicable. To help orient the reader, Figure~\ref{fig:conceptual_roadmap} provides a conceptual overview of this analytical framework and the role of each representation level.

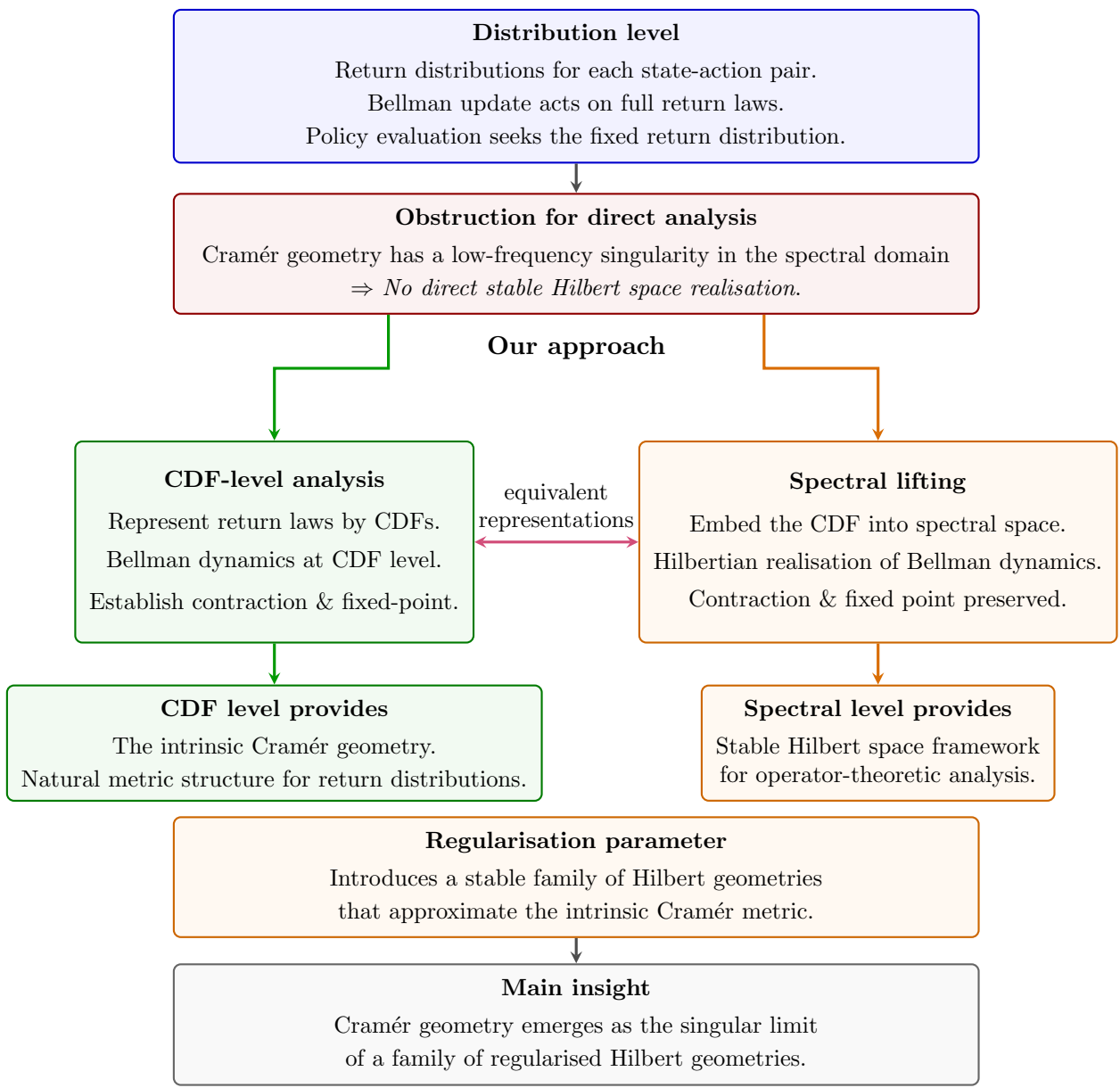
\begin{figure}[htbp]
\centering
\begin{tikzpicture}[
    box/.style={
        rectangle,
        rounded corners=3pt,
        draw=#1!80!black,
        line width=0.8pt,
        fill=#1!5,
        align=center,
        inner sep=6pt,
        font=\small
    },
    arrow/.style={
        ->,
        >=stealth,
        line width=1.2pt,
        draw=black!70
    }
]

\node[box=blue, minimum width=12cm, minimum height=1.8cm] (dist) at (0,7.0) {
   \textbf{Distribution level} \\[4pt]
Return distributions for each state-action pair. \\[2pt]
Bellman update acts on full return laws. \\[2pt]
Policy evaluation seeks the fixed return distribution.
};

\node[box=red!70!black, minimum width=12cm, minimum height=1.6cm] (obstruction) at (0,4.5) {
    \textbf{Obstruction for direct analysis} \\[4pt]
Cram\'er geometry has a low-frequency singularity in the spectral domain \\[2pt]
\(\Rightarrow\) \emph{No direct stable Hilbert space realisation}.
};

\draw[arrow] (dist.south) -- (obstruction.north);

\node[font=\bfseries] at (0,3.1) {Our approach};

\node[box=green!60!black, minimum width=5.2cm, minimum height=3.0cm] (cdf) at (-4.5,0.2) {
   \textbf{CDF-level analysis} \\[6pt]
Represent return laws by CDFs. \\[4pt]
Bellman dynamics at CDF level. \\[6pt]
Establish contraction \& fixed-point. 
};

\node[box=orange, minimum width=5.2cm, minimum height=3.0cm] (spec) at (4.5,0.2) {
    \textbf{Spectral lifting} \\[6pt]
Embed the CDF into spectral space. \\[4pt]
Hilbertian realisation of Bellman dynamics. \\[4pt]
Contraction \& fixed point preserved.
};

\draw[arrow, draw=green!60!black] 
    ($(obstruction.south) + (-2.8,0)$) -- ++(0,-0.8) -| (cdf.north);
\draw[arrow, draw=orange!85!black] 
    ($(obstruction.south) + (2.8,0)$) -- ++(0,-0.8) -| (spec.north);

\draw[<->, >=stealth, line width=1.2pt, draw=purple!70] 
    (cdf.east) -- node[above, font=\small, align=center] 
    {equivalent\\representations} (spec.west);

\node[box=green!60!black, minimum width=5.2cm, minimum height=1.5cm] (gain_cdf) at (-4.5,-2.8) {
    \textbf{CDF level provides} \\[4pt]
The intrinsic Cram\'er geometry. \\[2pt]
Natural metric structure for return distributions.
};

\node[box=orange, minimum width=5.2cm, minimum height=1.5cm] (gain_spec) at (4.5,-2.8) {
    \textbf{Spectral level provides} \\[4pt]
    Stable Hilbert space framework \\
    for operator-theoretic analysis.
};

\draw[arrow, draw=green!60!black] (cdf.south) -- (gain_cdf.north);
\draw[arrow, draw=orange!85!black] (spec.south) -- (gain_spec.north);

\node[box=orange, minimum width=12cm, minimum height=1.6cm] (reg) at (0,-4.8) {
    \textbf{Regularisation parameter} \\[4pt]
Introduces a stable family of Hilbert geometries \\[2pt]
that approximate the intrinsic Cram\'er metric.
};

\node[box=gray, minimum width=12cm, minimum height=1.8cm] (conclusion) at (0,-7.0) {
    \textbf{Main insight} \\[4pt]
Cram\'er geometry emerges as the singular limit \\[2pt]
of a family of regularised Hilbert geometries.
};

\draw[arrow] (reg.south) -- (conclusion.north);

\end{tikzpicture}

\caption{Conceptual roadmap of the paper. Starting from distributional Bellman dynamics, we identify a structural obstruction arising from the singular nature of the Cram\'er geometry, which prevents a direct stable Hilbert space representation. Our analysis therefore proceeds in two stages. First, we study the Bellman dynamics intrinsically at the level of CDFs, where contraction and fixed-point properties can be established under the Cram\'er geometry. Second, we construct a regularised spectral representation that realises these dynamics within a Hilbert space framework while preserving the underlying Bellman operator. The regularisation introduces a stable family of Hilbert geometries whose singular limit recovers the intrinsic Cram\'er metric.}
\label{fig:conceptual_roadmap}
\end{figure}
Our main contributions are as follows:
\begin{itemize}[leftmargin=*, itemsep=0.3em]
\item We present an intrinsic analysis of distributional Bellman dynamics at the raw CDF level under the Cram\'er geometry, establishing contraction, fixed-point existence, and convergence independently of any particular representation. This analysis reveals that the Bellman update acts affinely on distributions, with the Cram\'er geometry encoding the comparative structure that governs this affine action.

\item We construct a canonical Hilbert space envelope via exact conjugation that realises this intrinsic CDF geometry without altering the Bellman dynamics. For each fixed regularisation parameter, this yields a genuine Hilbert space in which the induced linear action of the Bellman update on CDF differences admits a bounded operator realisation; the native Cram\'er geometry is recovered monotonically in the zero-regularisation limit.

\item We establish the first spectral framework for analysing distributional Bellman updates by making explicit, via a functional-analytic interpretation, the operator structure underlying Cram\'er-based distributional Bellman dynamics. This lays the groundwork for systematic spectral and perturbation analyses that were previously inaccessible within the raw Cram\'er metric.
\end{itemize}

Together, these results establish a unified geometric and functional-analytic framework for distributional Bellman dynamics. By separating probabilistic dynamics from analytical representation, the proposed framework resolves a longstanding structural obstruction in Cram\'er-based DRL and provides a foundation for deeper theoretical investigations of distributional reinforcement learning.

\medskip
The remainder of the paper is organised as follows. Section~\ref{sec:prelims} introduces the probabilistic formulation of distributional policy evaluation and the Cramér geometry on cumulative distribution functions. Section~\ref{sec:methodology} outlines the three-stage representation strategy that guides our analysis, separating intrinsic dynamics from analytical coordinates. Section~\ref{sec:hilbert} constructs a regularised spectral Hilbert representation and establishes its exact correspondence with the intrinsic CDF geometry. Section~\ref{sec:bellman_embedding} develops the CDF-level Bellman analysis, derives the conjugate spectral representation, and studies fixed points and convergence. Sections~\ref{sec:relatedwork} and~\ref{sec:discussion} discuss related work and broader implications, respectively, and Section~\ref{sec:conclusion} concludes.

\section{Preliminaries and Analytical Background}
\label{sec:prelims}
To place Cram\'er-based distributional Bellman theory within a functional-analytic framework, we first make precise the probabilistic objects and geometric structures underlying distributional reinforcement learning. This section introduces the distributional formulation of policy evaluation, fixes notation for return random variables and return-distribution fields, and formalises the CDF geometry induced by the Cram\'er metric together with its natural working domain. We then review metrics on return distributions commonly used in distributional reinforcement learning and summarize the CDF-Fourier correspondence that underlies our subsequent spectral constructions.

\subsection{Bellman Dynamics and Distributional Policy Evaluation}
\label{subsec:mdp_prelims}
We recall the Markov decision process (MDP) framework~\citep{Suttonbook} and the distributional formulation of policy evaluation, fixing notation and clarifying the probabilistic objects underlying the CDF-level analysis developed later. In particular, we restrict attention throughout this section to policy evaluation under a fixed policy and do not consider policy optimisation or control. All statements in this subsection are purely probabilistic and do not rely on any spectral or Hilbert space constructions.

We consider a discounted infinite-horizon Markov decision process $<\mathcal S,\mathcal A, \mathcal P,\mathcal R, \gamma>,$ where $\mathcal S$ denotes the state space and $\mathcal A$ the action space, and $\gamma\in (0,1)$ is the discount factor governing the relative weight of future rewards.

For $s\in\mathcal S$ and $a\in\mathcal A$, $\mathcal P(\cdot \mid s,a)$ is the transition kernel on $\mathcal S$, and $\mathcal R(s,a,s')$ is the reward function on $\mathbb R$, specifying the one-step reward given the state-action-next-state triple $(s,a,s')$.

Given any action-generation mechanism, a \emph{trajectory} $\{S_t,A_t,R_t\}_{t\ge 0}$ is a stochastic process evolving according to $S_{t+1} \sim \mathcal P(\cdot \mid S_t,A_t)$, $R_t = \mathcal R( S_t,A_t,S_{t+1}),$ with randomness induced jointly by the transition and by the mechanism governing action selection, where $t$ denotes the discrete time index.

A (stationary) policy $\pi$ is a stochastic kernel on $\mathcal A$ given $\mathcal S$. When actions are generated according to $A_t\sim \pi(\cdot\mid S_t)$ for all $t\ge 0$, the joint process $\{S_t,A_t,R_t\}_{t\ge 0}$ evolves according to fixed transition, reward, and action-selection kernels and is Markov with respect to the state-action pair.

\medskip

\textbf{Return variable.}
Along a trajectory, define the discounted return variable
\begin{equation}
Z_t := \sum_{k=0}^{\infty} \gamma^k R_{t+k}, \qquad t \ge 0,
\label{eq:return_def}
\end{equation}
whenever the series is well defined. The return variable satisfies the one-step decomposition
\begin{equation}
Z_t = R_t + \gamma Z_{t+1}.
\label{eq:return_decomposition}
\end{equation}

For any action-generation mechanism governing actions after time $t$, define the state-action conditioned \emph{return random variable}
\begin{equation}
Z(s,a)
:= Z_t \ \text{conditioned on}\ (S_t=s,\ A_t=a),
\label{eq:state_action_return}
\end{equation}
which is well defined for any fixed $t\ge 0$ by time-homogeneity of the underlying process.

When actions are generated according to a stationary policy $\pi$, we define the policy-conditioned state-action \emph{return random variable} by
\begin{equation}
Z^\pi(s,a)
\;:=\;
Z_t \ \text{conditioned on}\ (S_t=s,\ A_t=a),
\quad
A_{t+k}\sim \pi(\cdot\mid S_{t+k}) \ \text{for all } k\ge 1,
\label{eq:policy_conditioned_return}
\end{equation}
which is well defined for any $t\ge 0$ by time homogeneity.

For notational convenience, we write $R(s,a)$ for the one-step reward random variable obtained as follows: first sample $S'\sim \mathcal  P(\cdot\mid s,a)$, and then calculate $R(s,a) =\mathcal R(s,a,S').$
Equivalently, $R(s,a)$ denotes a random variable with the same distribution as $R_t$ conditional on $(S_t,A_t)=(s,a)$.

Conditioning the one-step decomposition \eqref{eq:return_decomposition} on $(S_t,A_t)=(s,a)$ yields the Bellman recursion
\begin{equation}
Z(s,a) = R(s,a) + \gamma Z(S',A'),
\label{eq:rv_bellman_generic}
\end{equation}
where $S'\sim \mathcal P(\cdot\mid s,a)$ and $A'$ is generated by the action-selection mechanism governing future actions.

In particular, when actions after time $t$ are generated according to a policy $\pi$, the recursion \eqref{eq:rv_bellman_generic} specialises, by the definition \eqref{eq:policy_conditioned_return}, to
\begin{equation}
Z^\pi(s,a) = R(s,a) + \gamma Z^\pi(S',A'),
\label{eq:rv_bellman}
\end{equation}
with $S' \sim \mathcal P(\cdot\mid s,a)$ and $A'\sim \pi(\cdot\mid S')$.

\medskip

\textbf{Distributional Bellman operator.}
The random-variable Bellman equation \eqref{eq:rv_bellman_generic} expresses the one-step decomposition of returns: conditioned on a state-action pair $(s,a)$, the return decomposes into an immediate reward and a discounted future return evaluated at the next state-action pair. This identity is purely probabilistic and does not depend on a particular
policy.

When actions are generated according to a fixed policy $\pi$, conditioning \eqref{eq:rv_bellman_generic} yields the policy-specific recursion \eqref{eq:rv_bellman}, which serves as the starting point for distributional policy evaluation.

Distributional reinforcement learning transfers this recursion from random variables to their induced probability laws. While \eqref{eq:rv_bellman_generic} and \eqref{eq:rv_bellman} describe recursions at the level of random variables, they do not define a deterministic Bellman operator, as the update involves additional randomness through the next state-action pair~\citep{drlbook_bellemare_distributional_2023}. Passing to probability laws yields a deterministic operator acting on return distributions.

To express this update uniformly across state-action pairs, we introduce the notion of a return-distribution field. A \emph{return-distribution field} is a mapping
\begin{equation}
\mathcal Z : \mathcal S \times \mathcal A \longrightarrow \mathscr P(\mathbb R),
\label{def:field_z}
\end{equation}
which assigns to each state-action pair $(s,a)$ a probability law on $\mathbb R$. Here $\mathscr P(\mathbb R)$ denotes the set of all probability laws on $\mathbb R$. The field $\mathcal Z$ is related to the state-action return random variables by $Z(s,a) \sim \mathcal Z(s,a),\text{ or equivalently, } \mathcal Z(s,a) = \mathrm{Law}\!\left(Z(s,a)\right),$ where $\mathrm{Law}(\cdot)$ denotes taking the probability distribution of a random variable.

Given such a field $\mathcal Z$ and a fixed policy $\pi$, the \emph{distributional Bellman operator} $\mathcal T^\pi_{D}$ is defined pointwise by taking the law of the corresponding random-variable Bellman update~\citep{drlbook_bellemare_distributional_2023,c51,c51theory}. Specifically, for each $(s,a)\in\mathcal S\times\mathcal A$, we set
\begin{equation}
(\mathcal T^\pi_{D} \mathcal Z)(s,a)
:= \mathrm{Law}\!\left(
  R(s,a) + \gamma Z'
\right),
\label{eq:dist_bellman_operator_def}
\end{equation}
where $S'\sim \mathcal P(\cdot\mid s,a)$, $A'\sim \pi(\cdot\mid S')$, and, conditionally on $(S',A')$, the auxiliary random variable $Z'$ has law $\mathcal Z(S',A')$.

The \emph{policy-induced return-distribution field} $\mathcal Z^\pi := \{\mathcal Z^\pi(s,a)\}_{(s,a)\in\mathcal S\times\mathcal A}$ is defined pointwisely by $\mathcal Z^\pi(s,a) := \mathrm{Law}\!\left(Z^\pi(s,a)\right)$. By construction, taking laws on both sides of \eqref{eq:rv_bellman} yields the fixed-point~\citep{drlbook_bellemare_distributional_2023} equation
\begin{equation}
\mathcal Z^\pi = \mathcal T^\pi_{D} \mathcal Z^\pi.
\label{eq:dist_bellman_fp}
\end{equation}
Equation~\eqref{eq:dist_bellman_fp} characterises distributional policy evaluation under a fixed policy as a deterministic fixed-point problem on return-distribution fields.

\medskip
\textbf{Notational conventions.}
For notational simplicity, we use the shorthand $(s',a') \sim \mathcal P^\pi(\cdot\mid s,a)$ to denote the joint distribution obtained by first sampling $s'\sim \mathcal P(\cdot\mid s,a)$ and then sampling $a'\sim \pi(\cdot\mid s')$. Equivalently, $\mathcal P^\pi(\cdot\mid s,a)$ denotes the conditional law of $(S_{t+1},A_{t+1})$, i.e. $(S',A')$ discussed above in equations~\eqref{eq:rv_bellman_generic} and~\eqref{eq:rv_bellman}, given $(S_t,A_t)=(s,a)$ under policy $\pi$.
\subsection{Metrics on Return Distributions}
\label{subsec:metrics_return}
The analysis of distributional Bellman dynamics requires a choice of metric on spaces of return distributions~\citep{drlbook_bellemare_distributional_2023}. Unlike classical reinforcement learning, where value functions are compared pointwise, distributional reinforcement learning operates on probability laws and is therefore inherently metric-dependent~\citep{metric1,metric2}.

Throughout this paper, return distributions are represented through their \emph{cumulative distribution functions} (CDFs)~\citep{returncdf}. For a real-valued random variable $X$, its CDF is defined by~\citep{prob_intro}
\begin{equation}
F_X(x) := \mathbb P(X \le x), \qquad x\in\mathbb R .
\label{eq:cdf_def}
\end{equation}
When the law of $X$ is denoted by $\mathcal X$, we also write $F_{\mathcal X}:=F_X$. This representation reduces comparisons of distributions to comparisons of functions on $\mathbb R$.

Several metrics have been considered in the distributional reinforcement learning literature. Two prominent families are Wasserstein distances~\citep{qrdqn-Dabney18,iqn,fqf,dsacma} and CDF-based $L^p$ distances~\citep{drlbook_bellemare_distributional_2023}. Wasserstein metrics induce a weak topology with explicit moment control, but contraction of distributional Bellman updates under Wasserstein distances generally requires restrictive assumptions and does not hold in full generality~\citep{drlbook_bellemare_distributional_2023}.

An alternative is to compare distributions through their CDFs. Given two probability laws $P_1$ and $P_2$ on $\mathbb R$, the CDF-based $L^p$ distance~\citep{lpmetirc}is defined by
\[
\ell_p(F_{P_1},F_{P_2})
:=
\left(
\int_{\mathbb R}
\bigl|F_{P_1}(x)-F_{P_2}(x)\bigr|^p\,dx
\right)^{1/p},
\]
whenever finite. Such metrics are stronger than Wasserstein distances when well defined and allow one to work directly with function-level representations.

A distinguished case arises when $p=2$. In this setting, the induced distance coincides with the \emph{Cram\'er distance}~\citep{cramer1,cramer2} and is generated by an $L^2$ inner product on differences of CDFs. This inner-product structure~\citep{functional1} singles out the Cram\'er metric as a natural candidate for contraction-based analyses of distributional Bellman operators and as the geometric foundation for the framework developed in this paper.
\subsection{Cram\'er Geometry and its Working Domain}
\label{subsec:domain}
Having introduced CDFs as a canonical representation of return distributions and motivated the use of CDF-based metrics in Section~\ref{subsec:metrics_return}, we now focus on the specific geometry induced by the Cram\'er distance. This geometry plays a distinguished role in distributional reinforcement learning, as it yields a contraction property for the distributional Bellman operator while remaining strong enough to admit a functional-analytic description at the level of CDFs~\citep{drllinear,cramer-analysis}.

\paragraph{Cram\'er distance on CDFs.}
Let $P_1$ and $P_2$ be probability laws on $\mathbb R$, with cumulative distribution functions $F_{P_1}$ and $F_{P_2}$ defined as in \eqref{eq:cdf_def}. The Cram\'er distance between $P_1$ and $P_2$ is defined by
\begin{equation}
d_C(P_1,P_2)
\;:=\;
\ell_2(F_{P_1},F_{P_2})
\;=\;
\left(
\int_{\mathbb R}
\bigl(F_{P_1}(x)-F_{P_2}(x)\bigr)^2\,dx
\right)^{1/2},
\label{eq:cramer_def}
\end{equation}
where $\ell_2(\cdot,\cdot)$ denotes the standard $L^2(\mathbb R)$ distance~\citep{lpmetirc}.
That is, the Cram\'er distance coincides with the $L^2$ distance between cumulative distribution functions whenever this quantity is finite.

This identification endows the space of return distributions with an $L^2$-type geometry at the level of CDFs and provides the geometric setting in which contraction properties of the distributional Bellman operator are naturally formulated~\citep{drlstat}.

\paragraph{Working domain.}
The Cram\'er distance is not finite for arbitrary probability laws on $\mathbb R$. We therefore work on the natural domain on which the induced geometry is well defined. Let $\delta_0$ denote the Dirac measure at the origin, whose cumulative distribution function is given by
\begin{equation}
F_{\delta_0}(x) = \mathbbm{1}\{x \ge 0\}.
\label{eq:delta0_cdf}
\end{equation}
We define
\begin{equation}
\Gamma_F
:= \Bigl\{
        F_P:\mathbb R\to[0,1]
        \;:\;
        F_P \text{ is a cumulative distribution function and }
        d_C(P,\delta_0) < \infty
   \Bigr\},
\label{eq:GammaF_def}
\end{equation}
that is, the set of CDFs whose Cram\'er distance to $\delta_0$ is finite.
Equivalently,
\[
\Gamma_F
=
\Bigl\{
F_P :
F_P \text{ is a CDF and }
F_P - F_{\delta_0}\in L^2(\mathbb R)
\Bigr\}.
\]
Throughout this paper, unless otherwise stated, all return distributions are assumed to have cumulative distribution functions belonging to $\Gamma_F$. This domain is invariant under the transformations appearing in the distributional Bellman update, including reward translation and discount-induced scaling~\citep{c51}.

While the Cram\'er geometry provides a natural $L^2$ framework at the level of CDFs, it does not by itself furnish a Hilbert space realisation suitable for spectral analysis of Bellman transformations. To further analyse the structure induced by the Cram\'er geometry, it is convenient to pass from CDFs to an equivalent spectral representation, which we introduce next.

\subsection{CDF Representations and their Fourier Correspondence}
\label{subsec:cdf_fourier_background}
While the Cram\'er geometry~\citep{cramer-analysis} is formulated at the level of CDFs, an equivalent and often more transparent description is obtained in the Fourier domain~\citep{drlstat}. In particular, operations appearing in distributional Bellman updates admit simple algebraic representations in frequency space, which will play a central role in the spectral constructions developed later.

For a probability law $P$ on $\mathbb R$, let $X\sim P$ and denote its characteristic function~\citep{prob_intro} by
\[
\phi_P(\omega) := \mathbb E\!\left[e^{i\omega X}\right],
\qquad \omega\in\mathbb R.
\]
The characteristic function uniquely determines the law of $X$ and is classically related to the cumulative distribution function through Fourier analysis.

Throughout this paper, we adopt the symmetric Fourier transform~\citep{fourier} convention
\begin{equation}
\widehat f(\omega)
= \frac{1}{\sqrt{2\pi}}\int_{\mathbb R} f(x)e^{-i\omega x}\,dx,
\qquad
f(x)
= \frac{1}{\sqrt{2\pi}}\int_{\mathbb R}\widehat f(\omega)e^{i\omega x}\,d\omega,
\label{eq:fourier_convention}
\end{equation}
for which Plancherel's identity~\citep{plancherel} yields
\begin{equation}
\|f\|_{L^2(\mathbb R)}^2
=\int_{\mathbb R} | f(x)|^2\,dx=
\int_{\mathbb R} |\widehat f(\omega)|^2\,d\omega =\|\widehat f\|_{L^2(\mathbb R)}^2.
\label{eq:plancherel}
\end{equation}

This correspondence allows CDF-based distances, including the Cram\'er distance, to be reinterpreted as weighted $L^2$ quantities in frequency space. The resulting spectral viewpoint provides the technical foundation for the Hilbert space constructions introduced in Section~\ref{sec:hilbert}.
\subsection{Limitations of Hilbert Realisations under the Cram\'er Geometry}
\label{subsec:inner_product_limits}
At first sight, the Cram\'er geometry appears to be naturally compatible with a Hilbert space formulation. Indeed, on its working domain $\Gamma_F$, the Cram\'er distance coincides with the $L^2(\mathbb R)$ norm of differences of CDFs, suggesting that distributional Bellman dynamics might admit a direct Hilbert space realisation~\citep{functional1}.

However, this intuition is misleading. When the Cram\'er geometry is examined through the Fourier correspondence outlined in Section~\ref{subsec:cdf_fourier_background}, a fundamental obstruction becomes apparent. Although Plancherel's identity~\citep{plancherel} preserves $L^2$ norms under the Fourier transform, the Cram\'er metric corresponds in the frequency domain to a weighted $L^2$ structure with a singular low-frequency behaviour~\citep{drlstat}. In particular, the induced spectral weight diverges as $\omega\to 0$ (here $\omega$ refers to the frequency, and this description can be seen from equation~\eqref{eq:cramer_phi_formula}), producing a geometry that does not admit a stable Hilbert space realisation in spectral coordinates.

This singularity is benign for metric arguments carried out directly at the CDF level: contraction, invariance, and fixed-point properties of the distributional Bellman operator~\citep{cramer-analysis} are naturally formulated in CDF space and are unaffected by the low-frequency behaviour. However, when the Cram\'er geometry is examined through its Fourier representation, a structural limitation becomes apparent. The frequency-domain formula encodes the Cram\'er distance via differences of characteristic functions weighted by the singular factor $\omega^{-2}$ (see equation~\eqref{eq:cramer_phi_formula}). While this representation faithfully captures the metric geometry of CDF differences, the induced spectral structure is not itself a stable Hilbert space in which Bellman transformations admit a direct linear realisation.

As a result, while the Cram\'er metric provides an intrinsic and powerful geometric framework for analysing distributional Bellman dynamics at the level of CDFs, it does not by itself furnish a functional-analytic setting in which the induced linear action of the Bellman update on CDF differences can be studied as an operator on a Hilbert space. Resolving this tension therefore requires a modification at the level of geometry, rather than at the level of the Bellman operator itself.

Section~\ref{sec:hilbert} develops such a geometric resolution. We introduce a family of regularised Hilbert spaces that provide a stable linear realisation of the centred Cram\'er geometry, while recovering the intrinsic Cram\'er distance in a controlled limit.

\section{Methodological Framework of the Present Analysis}
\label{sec:methodology}
As outlined in the introduction, our analysis is guided by a representation strategy designed to separate the intrinsic structure of distributional Bellman dynamics from the choice of analytical coordinates. Rather than postulating Bellman operators directly in a Hilbert space, we proceed by first isolating all probabilistic and geometric properties at the level where the Bellman update is naturally defined (the CDF level), and only subsequently transporting these structures into a linear spectral setting.

Concretely, our methodology consists of the following three stages.

\begin{enumerate}
\item \textit{CDF-level formulation and analysis.} The distributional Bellman update is formulated and analysed entirely at the level of cumulative distribution functions. Working on the finite Cram\'er domain $\Gamma_F$ introduced in \eqref{eq:GammaF_def}, this representation provides an intrinsic $L^2$-type geometry on differences of CDFs relative to the reference distribution $F_{\delta_0}$. All contraction, invariance, and fixed-point arguments governing policy evaluation are established directly in this CDF space, without reference to any spectral or Hilbert space constructions.

\item \textit{Exact transport to a spectral Hilbert representation.} 
To enable a spectral and linear description of the induced action of the Bellman update on CDF differences, we introduce an exact transport that conjugates CDF representations into a regularised spectral Hilbert space. This transport is constructed so that the spectral operators arise solely by conjugation of the CDF-level Bellman update, and no modification of the Bellman dynamics themselves is introduced. As a result, all probabilistic and dynamical properties are inherited unchanged from the CDF level.

\item \textit{Geometric regularisation and recovery of the Cram\'er geometry.} The spectral representation is equipped with a one-parameter family of Hilbertian geometries that regularise the singular low-frequency structure induced by the Cram\'er metric. For each fixed regularisation parameter, this yields a stable Hilbert space in which the induced linear action of the Bellman update on centred CDF representations admits a well-defined operator realisation. At the same time, the induced Hilbert metrics are shown to converge monotonically to the intrinsic Cram\'er distance as the regularisation parameter tends to zero, ensuring that the spectral formulation remains geometrically faithful to the original CDF-based analysis.
\end{enumerate}

This three-stage construction provides a unified framework in which distributional Bellman dynamics can be analysed rigorously at the CDF level, while admitting an exact and geometrically controlled realisation in a Hilbert space for further structural and spectral investigations.

With this analytical strategy in place, we now proceed to the explicit
construction of the regularised Hilbert representation.

\section{Cram\'er Geometry and its Hilbert Envelope}
\label{sec:hilbert}
Section~\ref{sec:prelims} introduced CDFs as the primary representation of return distributions and fixed the Cram\'er metric as the associated geometry, defined intrinsically at the CDF level.

We now construct a Hilbert space realisation of the \emph{centred} CDF geometry that provides a functional-analytic representation of Cram\'er-based distributional Bellman dynamics. The centring arises intrinsically from the Cram\'er geometry itself, which measures differences between distributions rather than distributions in isolation. This realisation is purely representational: it does not modify the Bellman update or the underlying Cram\'er geometry. Rather, it provides an explicit spectral setting in which the induced linear action of the Bellman update on centred CDF differences can be realised as an operator on a Hilbert space, while preserving the original Cram\'er distances.

In spectral coordinates, the Cram\'er geometry corresponds to a weighted $L^2$ structure with a singular low-frequency behaviour. This singularity is intrinsic and cannot be removed by an equivalent renorming of the underlying function space, as it reflects the $\omega^{-2}$ weighting inherent in the Cram\'er metric. As a result, it obstructs a direct Hilbert space realisation based on standard spectral multipliers.

To resolve this obstruction, we introduce a one-parameter family of Hilbert spaces $\{\mathcal H_{\epsilon}\}_{\epsilon>0}$ obtained by a controlled regularisation of the Cram\'er geometry. The regularisation is chosen to be minimal and monotone: it affects only the singular low-frequency behaviour, leaves the representation of distributions unchanged, and admits an exact recovery of the Cram\'er metric in the zero-regularisation limit. For each fixed $\epsilon>0$ the geometry is Hilbertian and admits a stable spectral realisation; as $\epsilon\downarrow0$, the induced metric converges monotonically to the Cram\'er metric on its finite domain. In this sense, $\{\mathcal H_{\epsilon}\}_{\epsilon>0}$ forms a \emph{Hilbert envelope} of the Cram\'er geometry.
\subsection{The Cram\'er Distance and its Fourier Representation}
\label{subsec:cramer_bridge}
This subsection makes precise the spectral obstruction discussed above by deriving a Fourier-domain representation of the Cram\'er distance. Starting from the CDF-level identity
\[
d_C(P_1,P_2)
=
\|F_{P_1}-F_{P_2}\|_{L^2(\mathbb R)},
\]
we express the distance in frequency coordinates via Plancherel's theorem. The key step is an explicit formula for the Fourier transform of a difference of CDFs in terms of the corresponding characteristic functions, which reveals the singular low-frequency weighting that motivates the regularised Hilbert envelope constructed later.

Let $P_1$ and $P_2$ be probability laws with $F_{P_1},F_{P_2}\in\Gamma_F$, and define
\[
H(x):=F_{P_1}(x)-F_{P_2}(x).
\]
By definition of the Cram\'er distance and Plancherel's identity,
\[
d_C^2(P_1,P_2)
=
\|H\|_{L^2(\mathbb R)}^2
=
\|\widehat H\|_{L^2(\mathbb R)}^2
=
\int_{\mathbb R}|\widehat H(\omega)|^2\,d\omega.
\]

To relate $\widehat H$ to the underlying probability laws, we recall an explicit Fourier-analytic identity expressing $\widehat H$ in terms of the associated characteristic functions~\citep{cdftocf,integral}.

\begin{lemma}[Fourier transform of a CDF difference]
\label{lemma:fourier_cdf_difference}
Let $P_1,P_2$ be probability distributions with
$F_{P_1},F_{P_2}\in\Gamma_F$, and define
\[
H(x):=F_{P_1}(x)-F_{P_2}(x).
\]
Then for every $\omega\neq0$,
\begin{equation}
\widehat H(\omega)
=
\frac{\phi_{P_1}(-\omega)-\phi_{P_2}(-\omega)}
{i\omega\sqrt{2\pi}},
\label{eq:fourier_cdf_formula}
\end{equation}
where $\phi_{P}(\omega)
=
\mathbb E_{P}\!\left[e^{i\omega X}\right]$
denotes the characteristic function of $P$.
\end{lemma}

\begin{proof}
A complete proof is provided in Appendix~\ref{appendix:proof_fourier_cdf_difference}.
\end{proof}

Taking modulus in~\eqref{eq:fourier_cdf_formula} yields
\[
|\widehat H(\omega)|^2
=
\frac{|\phi_{P_1}(\omega)-\phi_{P_2}(\omega)|^2}
{2\pi\,\omega^2},
\qquad \omega\neq0.
\]
Substituting into the Plancherel representation gives the classical characteristic-function form of the Cram\'er distance~\citep{drlbook_bellemare_distributional_2023}:
\begin{equation}
d_C^2(P_1,P_2)
=
\frac{1}{2\pi}
\int_{\mathbb R}
\frac{|\phi_{P_1}(\omega)-\phi_{P_2}(\omega)|^2}
{\omega^2}\,d\omega.
\label{eq:cramer_phi_formula}
\end{equation}

Formula~\eqref{eq:cramer_phi_formula} reveals the origin of the zero-frequency singularity: spectral discrepancies are weighted by the unbounded factor $\omega^{-2}$. Although the cancellation $\phi_{P_1}(0)=\phi_{P_2}(0)=1$ ensures finiteness of the integral on the Cram\'er domain, the resulting spectral inner product is generated by a singular, non-uniform weight at low frequencies. As a consequence, the induced spectral geometry does not provide a stable Hilbert space structure in which the linear action of the Bellman update on centred CDF differences can be realised as a bounded operator. This structural limitation motivates the regularised Hilbert geometries introduced in the next subsection.
\subsection{A Spectrally Regularised Hilbert Envelope of the Cram\'er Geometry}
\label{subsec:cramer_spectral_space}
The Fourier representation~\eqref{eq:cramer_phi_formula} shows that the Cram\'er geometry corresponds, in spectral coordinates, to a singular low-frequency weighting. This observation identifies the precise obstruction to a direct Hilbert space formulation: while the geometry is well defined on the Cram\'er domain, it is not generated by a uniformly bounded spectral multiplier. Although this singularity is harmless for metric arguments carried out at the CDF level, it prevents the induced linear action of the Bellman update on centred CDF differences from being realised as a bounded operator in the associated spectral space. This identifies the precise spectral mechanism underlying the obstruction discussed in Section~\ref{subsec:inner_product_limits}.

Our approach resolves this obstruction at the level of geometry rather than by modifying return distributions or Bellman dynamics. Specifically, we introduce a one-parameter family of regularised Hilbert geometries that remain well defined for fixed regularisation while recovering the intrinsic Cram\'er geometry in the zero-regularisation limit. In this sense, the resulting spaces form a \emph{Hilbert envelope} of the Cram\'er geometry.

\paragraph{A regularised spectral Hilbert space.}
Motivated by the singular factor $\omega^{-2}$ in \eqref{eq:cramer_phi_formula}, we define a family of regularised spectral geometries by replacing this unbounded weight with $(\omega^2+\epsilon)^{-1}$. For each $\epsilon>0$, let
\[
\mathcal H_{\epsilon}
:= \Bigl\{
f \;\mid\;
\|f\|_{\mathcal H_{\epsilon}}^2
:= \frac{1}{2\pi}\!\int_{\mathbb R}
\frac{|\widehat f(\omega)|^2}{\omega^2+\epsilon}\,d\omega <\infty
\Bigr\},
\]
equipped with the inner product
\begin{equation}
\langle f,g\rangle_{\mathcal H_{\epsilon}}
:= \frac{1}{2\pi}\!\int_{\mathbb R}
\frac{\widehat f(\omega)\,\overline{\widehat g(\omega)}}{\omega^2+\epsilon}\,d\omega ,
\label{eq:cramer_reg_inner}
\end{equation}
where $\overline{\widehat g(\omega)}$ denotes the complex conjugate of
$\widehat g(\omega)$~\citep{functional1}. Via the Fourier transform, $\mathcal H_{\epsilon}$ is isometrically isomorphic to a weighted $L^2$ space on the frequency domain and therefore inherits a complete Hilbert structure.

\begin{proposition}
\label{prop:cramer_space_hilbert}
For every $\epsilon>0$, the space $\mathcal H_{\epsilon}$ equipped with the inner product~\eqref{eq:cramer_reg_inner} is a Hilbert space.
\end{proposition}

\begin{proof}
By construction, the Fourier transform defines an isometric isomorphism~\citep{functional1} between $\mathcal H_{\epsilon}$ and the weighted space
\[
L^2\!\left(\mathbb R, \frac{d\omega}{\omega^2+\epsilon}\right),
\]
which is complete since $\omega^2+\epsilon$ is strictly positive for all $\omega\in\mathbb R$.
\end{proof}

\paragraph{Spectral embedding of distributions.}
Having fixed the ambient geometry, we introduce a canonical representation of return distributions as elements of $\mathcal H_{\epsilon}$. Importantly, the embedding itself is independent of $\epsilon$; regularization enters only through the geometry of the space.

For each probability law $P$ on $\mathbb R$, define $\Phi(P)\in\mathcal H_{\epsilon}$ by specifying its Fourier transform as
\begin{equation}
\widehat{\Phi(P)}(\omega)
:= \phi_P(\omega)-1 ,
\label{eq:phi_embedding_def}
\end{equation}
where $\phi_P(\omega)=\mathbb E_{X \sim P}[e^{i\omega X}]$ denotes the characteristic function of $P$.

Since the characteristic function uniquely determines the law of a real-valued random variable and satisfies $\phi_P(0)=1$ for every probability measure $P$, the mapping $\phi_P \mapsto \phi_P-1$ is injective on the space of characteristic functions. Consequently, the embedding $\Phi$ retains the full distributional information of $P$, and no information is lost by removing the constant zero-frequency mode.

The subtraction of the constant mode plays a purely geometric role. It mirrors the CDF-level centring $F_P-F_{\delta_0}$ underlying the definition of the Cram\'er geometry and removes only a redundant degree of freedom associated with total mass. Although $\Phi(P)$ does not correspond to a probability density in the time domain, it provides a spectral representation of the underlying distribution, from which the original law can be uniquely recovered.

\paragraph{Well-definedness on the finite Cram\'er domain.}
The embedding~\eqref{eq:phi_embedding_def} is naturally well defined on the finite Cram\'er domain $\Gamma_F$. Fix $P$ with $F_P\in\Gamma_F$ and define
\[
H(x):=F_P(x)-F_{\delta_0}(x)\in L^2(\mathbb R).
\]
Applying Lemma~\ref{lemma:fourier_cdf_difference} with $P_1=P$ and $P_2=\delta_0$ yields, for $\omega\neq0$,
\[
\widehat H(\omega)
= \frac{\phi_P(-\omega)-1}{i\omega\sqrt{2\pi}} .
\]
Taking modulus and using $|\phi_P(-\omega)-1|=|\phi_P(\omega)-1|$ gives
\[
|\phi_P(\omega)-1|^2
= 2\pi\,\omega^2\,|\widehat H(\omega)|^2 .
\]
Consequently,
\[
\int_{\mathbb R}
\frac{|\phi_P(\omega)-1|^2}{\omega^2+\epsilon}\,d\omega
\le
2\pi
\int_{\mathbb R}|\widehat H(\omega)|^2\,d\omega
<\infty,
\]
where finiteness follows from Plancherel's identity. Thus $\Phi(P)\in\mathcal H_{\epsilon}$ for every $P$ with $F_P\in\Gamma_F$.

\begin{remark}[Regularisation by geometry, not embedding]
The mapping $P\mapsto \phi_P(\cdot)-1$ is fixed and injective, and therefore faithfully represents probability laws. Regularisation enters solely through the ambient Hilbert norm via the weight $(\omega^2+\epsilon)^{-1}$, which controls how spectral discrepancies are measured rather than which discrepancies are represented.
\end{remark}

\paragraph{Spectral distance identity.}
For any $P_1,P_2$ with $F_{P_1},F_{P_2}\in\Gamma_F$, the squared Hilbert distance between their embeddings satisfies
\begin{equation}
\|\Phi(P_1)-\Phi(P_2)\|_{\mathcal H_{\epsilon}}^2
= \frac{1}{2\pi}\!\int_{\mathbb R}
\frac{|\phi_{P_1}(\omega)-\phi_{P_2}(\omega)|^2}
{\omega^2+\epsilon}\,d\omega .
\label{eq:phi_embedding_distance}
\end{equation}
This expression is the regularised analogue of \eqref{eq:cramer_phi_formula}, obtained by replacing the singular weight
$\omega^{-2}$ with $(\omega^2+\epsilon)^{-1}$. The limiting relation back to the Cram\'er metric is established in Section~\ref{subsec:cramer_convergence}.
\subsection{Convergence to the Cram\'er Metric}
\label{subsec:cramer_convergence}
The purpose of this subsection is to establish that the regularised Hilbert geometries introduced in Section~\ref{subsec:cramer_spectral_space} approximate the intrinsic Cram\'er geometry. Specifically, we show that the metric induced by $\mathcal H_{\epsilon}$ converges monotonically to the Cram\'er distance as the regularisation parameter $\epsilon$ tends to zero, whenever the latter is finite.

Recall that for distributions $P_1,P_2$ with $F_{P_1},F_{P_2}\in\Gamma_F$, the squared Hilbert distance between their spectral embeddings is given by the regularised spectral discrepancy
\eqref{eq:phi_embedding_distance}. From a geometric viewpoint, the passage from the Cram\'er metric to $\mathcal H_{\epsilon}$ corresponds to replacing the singular weight $\omega^{-2}$ by the bounded weight $(\omega^2+\epsilon)^{-1}$ in frequency space.

Since $\phi_{P_1}(0)=\phi_{P_2}(0)=1$, the numerator in the spectral representation vanishes at zero frequency. As a result, the regularisation affects only the local behaviour of the geometry near $\omega=0$, while leaving the global spectral structure unchanged.

As $\epsilon\downarrow0$, the weights $(\omega^2+\epsilon)^{-1}$ increase monotonically to $\omega^{-2}$ for all $\omega\neq0$. On the finite Cram\'er domain, where the limiting integral is finite, this monotone increase allows one to pass to the limit under the integral sign.

\begin{proposition}[Spectral bridge to the Cram\'er metric]
\label{prop:cramer_spectral_bridge}
For any $P_1,P_2$ with $F_{P_1},F_{P_2}\in\Gamma_F$ and every $\epsilon>0$,
\[
\|\Phi(P_1)-\Phi(P_2)\|_{\mathcal H_{\epsilon}}^2
  = \frac{1}{2\pi}\!\int_{\mathbb R}
      \frac{|\phi_{P_1}(\omega)-\phi_{P_2}(\omega)|^2}{\omega^2+\epsilon}\,d\omega,
\]
and as $\epsilon\downarrow0$,
\[
\|\Phi(P_1)-\Phi(P_2)\|_{\mathcal H_{\epsilon}}
  \uparrow
  \left(
  \frac{1}{2\pi}
  \int_{\mathbb R}
     \frac{|\phi_{P_1}(\omega)-\phi_{P_2}(\omega)|^2}{\omega^2}\,d\omega
  \right)^{1/2}
  = d_C(P_1,P_2).
\]
\end{proposition}

\begin{proof}
A complete proof is provided in Appendix~\ref{appendix:proof_cramer_convergence}.
\end{proof}
The convergence result above shows that the regularised Hilbert geometries $\mathcal H_{\epsilon}$ approximate the intrinsic Cram\'er geometry on the finite Cram\'er domain $\Gamma_F$. This establishes the geometric consistency of the Hilbert envelope construction.

However, while distributional Bellman updates act on probability laws, their geometric and analytic structure is exposed only after passing to cumulative distribution function representations equipped with the Cram\'er metric. To analyse Bellman dynamics within the Hilbert framework without altering their structure, we therefore require a precise and norm-preserving transport between CDF representations and the spectral Hilbert space.

\subsection{CDF-Spectral Transport for Bellman Analysis}
\label{subsec:cdf_spectral_bridge}
This subsection introduces a transport mechanism that allows distributional Bellman dynamics, to be conjugated into the regularised spectral Hilbert space without modifying their intrinsic structure later in Section~\ref{sec:bellman_embedding}. The key step is an isometric identification between centered CDF representations and the spectral Hilbert space, which enables such a conjugation while preserving the underlying Cram\'er geometry.

We work throughout on the finite Cram\'er domain $\Gamma_F$ introduced in \eqref{eq:GammaF_def}.

\paragraph{Raw CDF parameterisation and centring.}
We represent return distributions through their CDFs. Raw CDFs serve only as a convenient parameterisation and carry no geometry. To introduce a Hilbert geometry compatible with the Cram\'er structure, we fix the reference CDF $F_{\delta_0}(x)=\mathbbm{1}\{x\ge 0\}$ and define the centring operation
\begin{equation}
C(F)(x) := F(x)-F_{\delta_0}(x).
\label{eq:centring_def}
\end{equation}
For a distribution $P$ with $F_P\in\Gamma_F$, we write
\begin{equation}
\mathcal U_C(P)(x) := C(F_P)(x) = F_P(x)-F_{\delta_0}(x).
\label{eq:UC_def}
\end{equation}

\paragraph{A centered CDF Hilbert space.}
For $\epsilon>0$, define
\[
\mathcal H^{\mathrm{cdf}}_\epsilon
 := \Bigl\{
       H\mid\;
       \|H\|^2_{\mathcal H^{\mathrm{cdf}}_\epsilon}
        := \int_{\mathbb R}
             \frac{\omega^2}{\omega^2+\epsilon}
             |\widehat H(\omega)|^2\,d\omega < \infty
    \Bigr\},
\]
equipped with the inner product
\[
\langle H_1,H_2\rangle_{\mathcal H^{\mathrm{cdf}}_\epsilon}
:=
\int_{\mathbb R}
\frac{\omega^2}{\omega^2+\epsilon}
\widehat H_1(\omega)\,
\overline{\widehat H_2(\omega)}\,d\omega ,
\]
where the overline denotes complex conjugation. The factor $\omega^2/(\omega^2+\epsilon)$ removes the constant mode and suppresses low-frequency components, yielding a well-posed Hilbert geometry on centred CDF representations that is compatible with the regularised spectral geometry introduced above.

\begin{lemma}[Membership of centered CDFs]
\label{lemma:centered_cdf_membership}
For every distribution $P$ with $F_P\in\Gamma_F$, we have
\[
\mathcal U_C(P) = F_P - F_{\delta_0} \in \mathcal H^{\mathrm{cdf}}_\epsilon .
\]
\end{lemma}

\begin{proof}
If $F_P\in\Gamma_F$, then by definition $\mathcal U_C(P)=F_P-F_{\delta_0}\in L^2(\mathbb R)$.
Plancherel's identity yields $\int_{\mathbb R}|\widehat{\mathcal U_C(P)}(\omega)|^2\,d\omega<\infty$.
Since $\omega^2/(\omega^2+\epsilon)\le 1$, we obtain
\[
\|\mathcal U_C(P)\|^2_{\mathcal H^{\mathrm{cdf}}_\epsilon}
= \int_{\mathbb R}\frac{\omega^2}{\omega^2+\epsilon}|\widehat{\mathcal U_C(P)}(\omega)|^2\,d\omega
\le \int_{\mathbb R}|\widehat{\mathcal U_C(P)}(\omega)|^2\,d\omega < \infty,
\]
which proves the claim.
\end{proof}

\paragraph{CDF-spectral isometric identification.}
Recall the spectral Hilbert space $\mathcal H_{\epsilon}$ from Section~\ref{subsec:cramer_spectral_space}. We now define a linear map that identifies the centered CDF Hilbert space $\mathcal H^{\mathrm{cdf}}_\epsilon$ with the spectral Hilbert space $\mathcal H_{\epsilon}$ in an isometric manner, preserving the corresponding Hilbert norms.

\begin{theorem}[CDF-spectral isometric isomorphism]
\label{thm:cdf_spectral_isomorphism}
Let $\epsilon>0$.
Define
\[
\mathcal U :
\mathcal H^{\mathrm{cdf}}_\epsilon \longrightarrow \mathcal H_{\epsilon}
\]
by prescribing its Fourier transform:
\begin{equation}
\widehat{(\mathcal U H)}(\omega)
 :=
 -\sqrt{2\pi}\, i \omega\overline{\widehat H(\omega)},
\qquad \omega\in\mathbb R.
\label{eq:U_definition}
\end{equation}
Then the following statements hold:
\begin{enumerate}
\item $\mathcal U$ is a conjugate-linear isometric isomorphism with $\|\mathcal U H\|_{\mathcal H_{\epsilon}}
 = \|H\|_{\mathcal H^{\mathrm{cdf}}_\epsilon},
$ where $
H\in\mathcal H^{\mathrm{cdf}}_\epsilon.$
\item For any $P_1,P_2$ with $F_{P_1},F_{P_2}\in\Gamma_F$,
\[
\mathcal U\bigl(F_{P_1}-F_{P_2}\bigr)
 = \Phi(P_1)-\Phi(P_2).
\]
In particular,
\[
\Phi(P) = \mathcal U\bigl(\mathcal U_C(P)\bigr),
\qquad \mathcal U_C(P)=F_P-F_{\delta_0}.
\]
\item $\mathcal U$ is onto $\mathcal H_{\epsilon}$ and admits
a bounded inverse $\mathcal U^{-1}:
\mathcal H_{\epsilon}\to\mathcal H^{\mathrm{cdf}}_\epsilon$ given by
\begin{equation}
\widehat{(\mathcal U^{-1} f)}(\omega)
 :=
 \frac{1}{\sqrt{2\pi}i\omega}\overline{\widehat f(\omega)},
\qquad \omega\in\mathbb R,
\label{eq:U_inverse}
\end{equation}
which also preserves norms:
\[
\|\mathcal U^{-1} f\|_{\mathcal H^{\mathrm{cdf}}_\epsilon}
 = \|f\|_{\mathcal H_{\epsilon}}.
\]
\end{enumerate}
\end{theorem}

\begin{proof}[Proof (sketch)]
The mapping and isometry properties follow from a direct computation of $\|\mathcal U H\|_{\mathcal H_{\epsilon}}$ using \eqref{eq:U_definition}, together with the definition of the $\mathcal H^{\mathrm{cdf}}_\epsilon$ and $\mathcal H_{\epsilon}$ norms. The identification with $\Phi(P_1)-\Phi(P_2)$ follows by combining Lemma~\ref{lemma:fourier_cdf_difference} with \eqref{eq:U_definition} and \eqref{eq:phi_embedding_def}. A complete proof is provided in Appendix~\ref{appendix:proof_cdf_spectral_isomorphism}.
\end{proof}

\begin{remark}[Real-linearity]
\label{remark:real_linearity}
Although $\mathcal U$ is conjugate-linear when viewed as an operator between complex Hilbert spaces, all constructions in this paper may equivalently be interpreted over real Hilbert spaces. In particular, when $\mathcal H^{\mathrm{cdf}}_\epsilon$ and $\mathcal H_{\epsilon}$ are regarded as real Hilbert spaces, $\mathcal U$ becomes a linear isometric isomorphism.
\end{remark}

\paragraph{Raw-to-spectral transport operator.}
Composing the centring operator $C$ with $\mathcal U$ yields a canonical transport from raw CDF representations to the regularised spectral space:
\begin{equation}
\label{eq:V_definition}
\mathcal V := \mathcal U \circ C .
\end{equation}
This map transports cumulative distribution functions into $\mathcal H_{\epsilon}$ via centring followed by an isometric identification at the Hilbert space level. On its range, the transport $\mathcal V$ admits a natural inverse given by
\begin{equation}
\label{eq:V_inverse_definition}
\mathcal V^{-1} := C^{-1} \circ \mathcal U^{-1},
\end{equation}
allowing representations to be transferred back and forth between the CDF and spectral settings.

\begin{proposition}[Canonical raw-to-spectral transport]
\label{prop:canonical_transport}
Fix $\epsilon>0$ and let $P_1,P_2$ be probability measures such that $F_{P_1},F_{P_2}\in\Gamma_F$. Then the transport operator $\mathcal V$ satisfies:
\begin{enumerate}
\item \emph{Consistency with spectral embedding.}
For any distribution $P$ with $F_P\in\Gamma_F$,
\begin{equation}
\label{eq:V_on_cdf}
\mathcal V(F_P)
= \mathcal U\bigl(F_P - F_{\delta_0}\bigr)
= \Phi(P).
\end{equation}

\item \emph{Distance consistency on $\Gamma_F$.}
The map $\mathcal V$ is injective on $\Gamma_F$, and for any $P_1,P_2$ as above,
\[
\|\Phi(P_1)-\Phi(P_2)\|_{\mathcal H_{\epsilon}}
=
\|F_{P_1}-F_{P_2}\|_{\mathcal H^{\mathrm{cdf}}_\epsilon}.
\]
In particular, $\mathcal V$ induces a distance-consistent embedding of $\Gamma_F$ into $\mathcal H_{\epsilon}$ when distances on $\Gamma_F$ are measured via the centered CDF norm $\|\cdot\|_{\mathcal H^{\mathrm{cdf}}_\epsilon}$.
\end{enumerate}
\end{proposition}

\begin{proof}
By definition, $\mathcal V=\mathcal U\circ C$, where the centring operator $C$ is injective on $\Gamma_F$ and $\mathcal U$ is an isometric isomorphism (Theorem~\ref{thm:cdf_spectral_isomorphism}). The identity \eqref{eq:V_on_cdf} follows immediately from the definitions, and the distance identity follows from the isometric property of $\mathcal U$ applied to centred CDF differences. A more detailed structural analysis of the transported subspaces, including span and closure properties, is provided in Appendix~\ref{appendix:cdf_spectral_structure}.
\end{proof}

\begin{remark}
Proposition~\ref{prop:canonical_transport} isolates the minimal geometric properties required to conjugate Bellman dynamics between the CDF and spectral representations. All operator-theoretic arguments in Section~\ref{sec:bellman_embedding} rely exclusively on these norm-consistency identities.
\end{remark}

Taken together, the constructions above realise the Cram\'er geometry through a \emph{Hilbert envelope}: for each $\epsilon>0$ the singular weight $\omega^{-2}$ is replaced by $(\omega^2+\epsilon)^{-1}$, yielding a stable Hilbert space amenable to linear and spectral analysis.

At the same time, the embedding $\widehat{\Phi(P)}(\omega)=\phi_P(\omega)-1$ retains the full spectral signature of the distribution, and regularisation enters only through the norm rather than the representation. As $\epsilon\downarrow0$, the induced metric converges monotonically back to the Cram\'er distance on the finite domain (Section~\ref{subsec:cramer_convergence}), so the family interpolates between a smooth, operator-friendly geometry and the native Cram\'er geometry.

\section{Exact Realisation of the Bellman Operator}
\label{sec:bellman_embedding}
We now make the structure of the distributional Bellman update explicit and realise it in a regularised Hilbert space without modifying the update. The term "realisation" is used in a representation-theoretic sense: no new dynamics are introduced, and all operators considered below are defined intrinsically at the CDF level and transported by exact conjugation. We proceed in two stages: a CDF-level formulation, where all probabilistic and metric properties are established, followed by an exact transport to the spectral Hilbert representation, which serves purely as an analytical realisation.

First, we formulate the Bellman update at the level of CDFs. In Section~\ref{subsec:assumptions} and Section~\ref{subsec:bellman_unified_order}, the update is decomposed into three primitive operations (reward translation, discount-induced scaling, and conditional expectation) acting on CDF fields. This formulation is intrinsic and well defined on the finite Cram\'er domain $\Gamma_F$.

Second, we transport the CDF-level dynamics into the regularised spectral Hilbert space by conjugation with the CDF-spectral transport from Section~\ref{subsec:cdf_spectral_bridge}. The resulting spectral Bellman operator, defined in Section~\ref{subsec:intertwining}, acts on the invariant set $\mathrm{range}(\mathbb V)\subset\mathcal X^{\mathrm{spec}}$ and is not postulated independently. Rather, it is defined uniquely as the conjugate~\citep{operator-conjugate} image of the CDF-level operator under a norm-preserving identification, so that no additional operator-theoretic assumptions are introduced at the spectral level.

Concretely, each return distribution $P$ is represented by its spectral embedding $\Phi(P)=\mathcal V(F_P)$, where $\mathcal V=\mathcal U\circ C$ is the raw-to-spectral transport and $C(F_P)=F_P-F_{\delta_0}$ denotes centring. All spectral operators in this section act on embeddings of admissible CDF fields via the lifted map $\mathbb V$ (Section~\ref{subsec:intertwining}).

This construction yields a faithful Hilbert space realisation of the distributional Bellman update in a representation-theoretic sense: contraction and fixed-point properties are established intrinsically in the Cram\'er geometry on $(\mathcal X^{\mathrm{cdf}}, d_{\mathrm{Cr}})$ (Section~\ref{subsec:cramer_contraction}), while the spectral representation provides a linear operator setting for the induced action of the Bellman update on centred CDF representations, expressed through exact conjugation. The $\epsilon\downarrow0$ connection to the classical Cram\'er metric follows from Section~\ref{subsec:cramer_convergence}.

\subsection{Assumptions and CDF-Level Setup}
\label{subsec:assumptions}
In this subsection, we make the CDF-level formulation of the distributional Bellman update precise. All analytical structure is imposed directly on CDF representations of return-distribution fields, without reference to any spectral or Hilbert space constructions. The resulting CDF-level operators will later be transported into the spectral Hilbert space via the embedding developed in Section~\ref{subsec:cdf_spectral_bridge}.
\paragraph{Admissible CDF class.}
We work on the finite Cram\'er domain $\Gamma_F$ introduced by~\eqref{eq:GammaF_def} in Section~\ref{sec:prelims}. Elements of $\Gamma_F$ are cumulative distribution functions whose distance to the reference CDF $F_{\delta_0}$ is finite under the Cram\'er geometry. Via the identification $F \mapsto F - F_{\delta_0}$, the set $\Gamma_F$ may be viewed as an affine subset of $L^2(\mathbb R)$.

\paragraph{Return-distribution fields and their CDF representations.}
Recall from Section~\ref{subsec:mdp_prelims} that a return-distribution field is a mapping $\mathcal Z : \mathcal S \times \mathcal A \rightarrow \mathscr P(\mathbb R)$ as defined in~\eqref{def:field_z},  assigning to each state-action pair $(s,a)$ the conditional law of a return random variable.
Throughout this section, we consider distribution fields $\mathcal Z$ such that their pointwise CDFs belong to the admissible class $\Gamma_F$.

To each such distribution field $\mathcal Z$, we associate its CDF representation
\begin{equation}
    F_{\mathcal Z} \in \mathcal X^{\mathrm{cdf}},
\qquad
F_{\mathcal Z}(s,a)(x)
:= F_{\mathcal Z(s,a)}(x),
\label{def:cdf_field}
\end{equation}
where $\mathcal X^{\mathrm{cdf}}$ denotes the space of state-action indexed CDF fields defined below.

\paragraph{State-action indexed CDF fields.}
We define the space of CDF-valued fields by
\begin{equation}
    \mathcal X^{\mathrm{cdf}}
:= \bigl\{
      F_{\mathcal Z} \mid
      \mathcal Z : \mathcal S \times \mathcal A \to \mathscr{P}(\mathbb R),
      \;
      \qquad \text{s.t. } F_{\mathcal Z(s,a)} \in \Gamma_F \text{ for all $(s,a)$}\ 
   \bigr\}.
   \label{def:x_cdf}
\end{equation}
Elements of $\mathcal X^{\mathrm{cdf}}$ are treated abstractly as mappings $(s,a)\mapsto F_{\mathcal Z(s,a)}$, without assuming \emph{a priori} that an arbitrary CDF field arises from an underlying stochastic process. The identification with policy-induced return distributions will be established later via an intertwining and fixed-point argument.

\paragraph{Metric structure.}
For two return-distribution fields $\mathcal Z_1$ and $\mathcal Z_2$ whose CDF representations belong to $\mathcal X^{\mathrm{cdf}}$, we define the state-action supremum Cram\'er metric
\begin{equation}
d_{\mathrm{Cr}}(\mathcal Z_1,\mathcal Z_2)
:=
\sup_{(s,a)\in\mathcal S\times\mathcal A}
\left(
\int_{\mathbb R}
\bigl(
F_{\mathcal Z_1(s,a)}(x)
-
F_{\mathcal Z_2(s,a)}(x)
\bigr)^2
\,dx
\right)^{1/2}.
\label{eq:cramer_metric_state_action}
\end{equation}
Since $F_{\mathcal Z_i(s,a)} \in \Gamma_F$ for all $(s,a)$, the integral is finite and $d_{\mathrm{Cr}}$ is well defined. Moreover, this metric coincides with the classical Cram\'er distance between the corresponding return distributions, as established in Section~\ref{sec:hilbert}.

For analytical convenience later in Section~\ref{subsec:cramer_contraction}, we will overload this metric notation by writing $d_{\mathrm{Cr}}(F_{\mathcal Z_1},F_{\mathcal Z_2})$ when no ambiguity arises.

\medskip
We now state the standing assumptions underlying the CDF-level formulation of the distributional Bellman update. The first two assumptions ensure that the CDF-level Bellman operator is well defined on the CDF field domain $\mathcal X^{\mathrm{cdf}}$, while the third assumption guarantees that the policy-induced return distributions themselves lie in the admissible domain, so that distributional policy evaluation is well posed.

\begin{assumption}[Bounded rewards]
\label{assump:reward_bound}
There exists $R_{\max} < \infty$ such that $|R(s,a)| \le R_{\max}$ almost surely for all $(s,a) \in \mathcal S \times \mathcal A$.
\end{assumption}

\begin{assumption}[Measurability of the transition operator]
\label{assump:transition}
For any return-distribution field $\mathcal Z$ with $F_{\mathcal Z} \in \mathcal X^{\mathrm{cdf}}$ and each $x \in \mathbb R$, the mapping
\[
(s,a) \longmapsto
\mathbb E_{(s',a') \sim \mathcal P^\pi(\cdot \mid s,a)}
\bigl[\, F_{\mathcal Z(s',a')}(x) \,\bigr]
\]
is measurable.
\end{assumption}

\begin{assumption}[Admissible return CDFs]
\label{assump:gammaF_domain}
For the fixed policy $\pi$, the policy-induced return-distribution field
$\mathcal Z^\pi$ satisfies
\[
F_{\mathcal Z^\pi(s,a)} \in \Gamma_F
\qquad
\text{for all } (s,a)\in\mathcal S\times\mathcal A.
\]
\end{assumption}

Taken together, these assumptions ensure that the CDF-level Bellman operator is well defined on the admissible CDF field domain $\mathcal X^{\mathrm{cdf}}$, and that policy evaluation can be analysed as a contraction mapping problem in the Cram\'er geometry.

\paragraph{Spectral correspondence (preview).}
For later use, we note that these CDF-level constructions admit an exact spectral realisation via the transport operator
\[
\mathcal V : \Gamma_F \longrightarrow \mathcal H_{\epsilon},
\qquad
\mathcal V(F)
= \mathcal U(F - F_{\delta_0}),
\]
introduced in Section~\ref{subsec:cdf_spectral_bridge}. For probability measures $P$ with $F_P \in \Gamma_F$, this recovers the regularised spectral embedding $\Phi(P)=\mathcal V(F_P)$. The precise operator-level correspondence and its consequences for Bellman dynamics are developed in Sections~\ref{subsec:intertwining} and Section~\ref{subsec:spectral_fixed_point}.

\subsection{Operators Realising the Bellman Update}
\label{subsec:bellman_unified_order}
In this subsection we make the CDF-level Bellman update explicit by identifying the primitive operators that realise its action on CDFs.

All operators act on CDF representations of return-distribution fields, namely on elements $F_{\mathcal Z}\in\mathcal X^{\mathrm{cdf}}$ introduced in Section~\ref{subsec:assumptions}, and are studied under the supremum Cram\'er metric $d_{\mathrm{Cr}}$.

\vspace{0.3em}
We begin by defining three primitive operators corresponding to reward translation, geometric discounting, and conditional expectation. Throughout, $\mathcal Z$ denotes a generic return-distribution field whose CDF representation $F_{\mathcal Z}$ belongs to $\mathcal X^{\mathrm{cdf}}$.
\paragraph{Reward translation operator.}
For each $(s,a)\in\mathcal S\times\mathcal A$, let $R(s,a)$ denote the immediate reward random variable (Section~\ref{subsec:mdp_prelims}).
We define the reward translation operator $\mathcal S_R$ by its pointwise action on CDF representations:
\begin{equation}
(\mathcal S_R F_{\mathcal Z})(s,a)(x)
 :=
 \mathbb E_{r\sim R(s,a)}
 \bigl[
   F_{\mathcal Z(s,a)}(x-r)
 \bigr].
\label{eq:reward_translation_operator}
\end{equation}
At the level of probability measures, this operator corresponds to translation of the return distribution by the immediate reward.

\begin{lemma}[Invariance under reward translation]
\label{lem:invariance_SR}
Under Assumption~\ref{assump:reward_bound}, if $F_{\mathcal Z}\in\mathcal X^{\mathrm{cdf}}$, then $\mathcal S_R F_{\mathcal Z}\in\mathcal X^{\mathrm{cdf}}$. Equivalently, the admissible CDF class $\Gamma_F$ is invariant under $\mathcal S_R$.
\end{lemma}

\begin{proof}
Fix $(s,a)$ and write $f := F_{\mathcal Z(s,a)}\in\Gamma_F$, so that $f-F_{\delta_0}\in L^2(\mathbb R)$. For each fixed reward realisation $r$, the translation $f(\cdot)\mapsto f(\cdot-r)$ is an $L^2(\mathbb R)$--isometry. By Assumption~\ref{assump:reward_bound}, the reward random variable $R(s,a)$ is almost surely bounded, so translations of the reference CDF $F_{\delta_0}$ remain at finite $L^2$ distance from $F_{\delta_0}$. Taking expectation over $r$ therefore preserves finiteness of the Cram\'er distance, and hence membership in $\Gamma_F$.
\end{proof}

\paragraph{Discount scaling operator.}
For a fixed discount factor $\gamma\in(0,1)$, we define the discount scaling operator $\mathcal D_\gamma$ by
\begin{equation}
(\mathcal D_\gamma F_{\mathcal Z})(s,a)(x)
 :=
 F_{\mathcal Z(s,a)}\!\left(\frac{x}{\gamma}\right).
\label{eq:discount_scaling_operator}
\end{equation}
This represents the pushforward of the underlying return distribution under the map $x\mapsto \gamma x$.

\begin{lemma}[Invariance under discount scaling]
\label{lem:invariance_Dgamma}
If $F_{\mathcal Z}\in\mathcal X^{\mathrm{cdf}}$, then $\mathcal D_\gamma F_{\mathcal Z}\in\mathcal X^{\mathrm{cdf}}$. Equivalently, the admissible CDF class $\Gamma_F$ is invariant under $\mathcal D_\gamma$.
\end{lemma}

\begin{proof}
Fix $(s,a)$ and write $f := F_{\mathcal Z(s,a)}\in\Gamma_F$, so that $f - F_{\delta_0} \in L^2(\mathbb R)$.

By definition,
\[
(\mathcal D_\gamma f)(x) = f(x/\gamma).
\]
Using the identity $F_{\delta_0}(\gamma u)=F_{\delta_0}(u)$ for all $\gamma>0$, we compute
\[
\int_{\mathbb R}
\bigl( f(x/\gamma) - F_{\delta_0}(x) \bigr)^2\,dx
=
\gamma
\int_{\mathbb R}
\bigl( f(u) - F_{\delta_0}(\gamma u) \bigr)^2\,du=
\gamma
\int_{\mathbb R}
\bigl( f(u) - F_{\delta_0}(u) \bigr)^2\,du
< \infty,
\]
where we used the change of variables $u=x/\gamma$. Hence $(\mathcal D_\gamma f) - F_{\delta_0}\in L^2(\mathbb R)$ and $\mathcal D_\gamma f\in\Gamma_F$.
\end{proof}

\paragraph{Conditional expectation operator.}
Let $\pi$ be a fixed policy and $\mathcal P^\pi(\cdot\mid s,a)$ the induced transition kernel over successor state-action pairs (as introduced in Section~\ref{subsec:mdp_prelims}). We define the conditional expectation operator $\mathcal C^\pi$ by
\begin{equation}
(\mathcal C^\pi F_{\mathcal Z})(s,a)(x)
 :=
 \mathbb E_{(s',a')\sim \mathcal P^\pi(\cdot\mid s,a)}
 \bigl[
   F_{\mathcal Z(s',a')}(x)
 \bigr].
\label{eq:conditional_expectation_operator}
\end{equation}
This operator aggregates successor return distributions by averaging their CDFs pointwise.

\begin{lemma}[Invariance under conditional expectation]
\label{lem:invariance_Cpi}
Under Assumption~\ref{assump:transition}, if $F_{\mathcal Z}\in\mathcal X^{\mathrm{cdf}}$, then $\mathcal C^\pi F_{\mathcal Z}\in\mathcal X^{\mathrm{cdf}}$. Equivalently, the admissible CDF class $\Gamma_F$ is invariant under $\mathcal C^\pi$.
\end{lemma}

\begin{proof}
Fix $(s,a)$. By Assumption~\ref{assump:transition}, for each $x\in\mathbb R$ the mapping $(s,a)\mapsto(\mathcal C^\pi F_{\mathcal Z})(s,a)(x)$ is measurable. Moreover, $(\mathcal C^\pi F_{\mathcal Z})(s,a)$ is a convex combination of successor CDFs in $\Gamma_F$.
Since $L^2(\mathbb R)$ is a linear space and the $L^2$ norm is convex, such averaging preserves finite $L^2$ distance to $F_{\delta_0}$. Hence $(\mathcal C^\pi F_{\mathcal Z})(s,a)\in\Gamma_F$.
\end{proof}
\paragraph{CDF-level Bellman operator.}
By composing the three primitive operators above, we define the CDF-level Bellman operator
\begin{equation}
\mathcal T^{\pi}_{\mathrm{cdf}}
 :=
 \mathcal C^\pi \circ \mathcal D_\gamma \circ \mathcal S_R .
\label{eq:cdf_bellman_operator}
\end{equation}
Equivalently, for any return-distribution field $\mathcal Z$ with $F_{\mathcal Z}\in\mathcal X^{\mathrm{cdf}}$, the Bellman update admits the pointwise representation
\begin{equation}
(\mathcal T^\pi_{\mathrm{cdf}} F_{\mathcal Z})(s,a)(x)
 =
 \mathbb E_{r,(s',a')}
 \Bigl[
   F_{\mathcal Z(s',a')}\!\left(\frac{x-r}{\gamma}\right)
 \Bigr],
\label{eq:Tpi_cdf_def}
\end{equation}
where $r\sim R(s,a)$ and $(s',a')\sim \mathcal P^\pi(\cdot\mid s,a)$.

\begin{proposition}[Invariance of the admissible CDF class]
\label{lem:GammaF_closed}
Under Assumptions~\ref{assump:reward_bound}--\ref{assump:transition}, the CDF Bellman operator $\mathcal T^\pi_{\mathrm{cdf}}$ maps $\mathcal X^{\mathrm{cdf}}$ into itself. Equivalently, the admissible CDF class $\Gamma_F$ is invariant under $\mathcal T^\pi_{\mathrm{cdf}}$.
\end{proposition}

\begin{proof}
By Lemmas~\ref{lem:invariance_SR}, \ref{lem:invariance_Dgamma}, and \ref{lem:invariance_Cpi}, each primitive operator preserves $\mathcal X^{\mathrm{cdf}}$ under the stated assumptions. Since $\mathcal T^\pi_{\mathrm{cdf}}$ is their composition, it also maps $\mathcal X^{\mathrm{cdf}}$ into itself.
\end{proof}

\paragraph{Preview: spectral correspondence and policy evaluation.}
Proposition~\ref{lem:GammaF_closed} shows that the CDF-level Bellman operator preserves the admissible CDF field domain $\mathcal X^{\mathrm{cdf}}$ under Assumptions~\ref{assump:reward_bound} and~\ref{assump:transition}. The additional Assumption~\ref{assump:gammaF_domain} ensures that the policy-induced return-distribution field itself lies in this space, so that distributional policy evaluation is well posed.

The CDF-level Bellman operator further admits an exact conjugate representation in the regularised spectral Hilbert space via pointwise transport with the map $\mathcal V$ introduced in Section~\ref{subsec:cdf_spectral_bridge}. The construction of the lifted transport on CDF fields and the resulting spectral Bellman operator is developed in Section~\ref{subsec:intertwining}.
\subsection{Intertwining of Bellman Operators}
\label{subsec:intertwining}
This subsection establishes an exact correspondence between the Bellman update across three representation levels: return-distribution fields, their CDF representations, and their regularised spectral embeddings. Rather than introducing new dynamics at each level, we show that all three formulations are related by precise intertwining identities. In particular, the spectral Bellman operator is defined uniquely as the conjugate image of the CDF-level Bellman update, so that no additional operator-theoretic structure is introduced beyond the intrinsic CDF-level dynamics.
We begin by verifying that the distributional Bellman update defined on return-distribution fields is represented at the level of CDF fields. This step isolates the Bellman dynamics on cumulative distribution functions, independently of any spectral or Hilbert space constructions.

\begin{lemma}[Distribution-CDF intertwining]
\label{lemma:raw_cdf_intertwine}
For any return-distribution field $\mathcal Z$ with $F_{\mathcal Z}\in\mathcal X^{\mathrm{cdf}}$, the following identity holds:
\begin{equation}
F_{\mathcal T^\pi_{D}\mathcal Z}
= \mathcal T^\pi_{\mathrm{cdf}}\,F_{\mathcal Z},
\label{eq:raw_cdf_identity_field}
\end{equation}
as an equality in $\mathcal X^{\mathrm{cdf}}$. Equivalently, for all $(s,a)\in\mathcal S\times\mathcal A$ and $x\in\mathbb R$,
\[ F_{(\mathcal T^\pi_{D}\mathcal Z)(s,a)}(x)
= (\mathcal T^\pi_{\mathrm{cdf}}F_{\mathcal Z})(s,a)(x). \]
\end{lemma}

\begin{proof}
Fix $(s,a)\in\mathcal S\times\mathcal A$ and $x\in\mathbb R$. By definition of $\mathcal T^\pi_{D}$ (Section~\ref{subsec:mdp_prelims}),
\[ (\mathcal T^\pi_{D}\mathcal Z)(s,a)
= \mathrm{Law}\!\left(R(s,a)+\gamma Z'\right), \]
where $(s',a')\sim \mathcal P^\pi(\cdot\mid s,a)$ and, conditionally on $(s',a')$, $Z'$ has law $\mathcal Z(s',a')$.
Therefore,
\begin{align*}
F_{(\mathcal T^\pi_{D}\mathcal Z)(s,a)}(x)
&= \mathbb P\!\left(R(s,a)+\gamma Z' \le x\right) \\
&= \mathbb E_{r\sim R(s,a)}\Bigl[
     \mathbb P\!\left(Z' \le \frac{x-r}{\gamma} \right)
 \Bigr] \\
&= \mathbb E_{r,(s',a')}
   \Bigl[
     F_{\mathcal Z(s',a')}\!\left(\tfrac{x-r}{\gamma}\right)
   \Bigr] \\
&= (\mathcal T^\pi_{\mathrm{cdf}}F_{\mathcal Z})(s,a)(x),
\end{align*}
which proves~\eqref{eq:raw_cdf_identity_field}.
\end{proof}
\paragraph{Spectral field domain and lifted transport.}
Having established that the Bellman update is exact and well defined at the level of CDF fields, we now introduce the lifted spectral representation that will be used to realise these dynamics by conjugacy.

Fix $\epsilon>0$ and consider the ambient space of $\mathcal H_\epsilon$-valued fields
\begin{equation}
\label{eq:X_spec_definition}
\mathcal X^{\mathrm{spec}}
:=
\Bigl\{
\mu:\mathcal S\times\mathcal A \to \mathcal H_{\epsilon}
\;\mid\;
\mu \text{ is measurable and }
\|\mu\|_{\infty,\epsilon}<\infty
\Bigr\},
\end{equation}
equipped with the supremum norm
\begin{equation}
\label{eq:X_spec_norm}
\|\mu\|_{\infty,\epsilon}
:=
\sup_{(s,a)\in\mathcal S\times\mathcal A}
\|\mu(s,a)\|_{\mathcal H_{\epsilon}}.
\end{equation}
This ambient norm is introduced only to specify a well-behaved container for spectral fields. The metric structure relevant for policy-evaluation analysis will instead be induced from the intrinsic Cram\'er geometry on the admissible spectral subset defined below.

Recall from Section~\ref{subsec:cdf_spectral_bridge} that the raw-to-spectral transport $\mathcal V:\Gamma_F\to\mathcal H_{\epsilon}$ provides an injective representation of admissible CDFs. We lift this mapping pointwise to CDF fields by defining
\begin{equation}
\label{eq:lifted_V_definition}
\mathbb V:\mathcal X^{\mathrm{cdf}}\longrightarrow \mathcal X^{\mathrm{spec}},
\qquad
(\mathbb V F_{\mathcal Z})(s,a)
:= \mathcal V\!\bigl(F_{\mathcal Z(s,a)}\bigr).
\end{equation}
By construction, $\mathbb V$ maps $\mathcal X^{\mathrm{cdf}}$ injectively into $\mathcal X^{\mathrm{spec}}$. We therefore identify $\mathcal X^{\mathrm{cdf}}$ with its image $\mathrm{range}(\mathbb V)\subset\mathcal X^{\mathrm{spec}}$ and restrict all spectral Bellman dynamics to this admissible subset. The inverse mapping $\mathbb V^{-1}$ is well defined on $\mathrm{range}(\mathbb V)$ by pointwise application of $\mathcal V^{-1}$.

In Section~\ref{subsec:spectral_fixed_point}, we endow $\mathrm{range}(\mathbb V)$ with the induced Cram\'er metric (Definition~\ref{def:induced_metric}), under which well-posedness of policy evaluation is inherited by conjugacy.
\paragraph{Spectral Bellman operator by conjugation.}
The identities above show that the Bellman update is already well defined and exact at the level of CDF fields. We now introduce a spectral representation of this dynamics, not as an independently postulated operator, but as an exact conjugate of the CDF-level Bellman update.

Using the lifted raw-to-spectral transport $\mathbb V$, we define the spectral Bellman operator by conjugation~\citep{operator-conjugate}:
\begin{equation}
\label{eq:Tpi_spec_function}
\mathcal T^{\pi}
:= \mathbb V \circ \mathcal T^\pi_{\mathrm{cdf}} \circ \mathbb V^{-1},
\end{equation}
acting on the invariant subset $\mathrm{range}(\mathbb V)\subset\mathcal X^{\mathrm{spec}}$. By construction, the spectral dynamics inherits its structure entirely from the CDF-level Bellman operator.

\begin{proposition}[Spectral intertwining]
\label{prop:intertwining_full}
For any return-distribution field $\mathcal Z$ with $F_{\mathcal Z}\in\mathcal X^{\mathrm{cdf}}$, the following identity holds:
\begin{equation}
\mathbb V\!\left(
  F_{\mathcal T^\pi_{D}\mathcal Z}
\right)
=
\mathcal T^{\pi}\,
\mathbb V\!\left(F_{\mathcal Z}\right).
\label{eq:intertwining_statement_full_field}
\end{equation}
\end{proposition}

\begin{proof}
Lemma~\ref{lemma:raw_cdf_intertwine} gives $F_{\mathcal T^\pi_{D}\mathcal Z} = \mathcal T^\pi_{\mathrm{cdf}}F_{\mathcal Z}$. Applying $\mathbb V$ to both sides and using \eqref{eq:Tpi_spec_function} yields the claim.
\end{proof}

\begin{remark}[Three-level equivalence]
\label{rem:three_level_equivalence}
Lemma~\ref{lemma:raw_cdf_intertwine} and Proposition~\ref{prop:intertwining_full} together establish that the distributional, CDF-level, and spectral formulations of the Bellman update are exactly equivalent. Specifically, the diagram
\[
\begin{array}{ccccc}
\mathcal Z & \xrightarrow{\mathrm{CDF}} & F_{\mathcal Z} & \xrightarrow{\mathbb V} & \mu \\[4pt]
\downarrow & & \downarrow & & \downarrow \\[4pt]
\mathcal T^\pi_{D}\mathcal Z & \xrightarrow{\mathrm{CDF}} & \mathcal T^\pi_{\mathrm{cdf}}F_{\mathcal Z} & \xrightarrow{\mathbb V} & \mathcal T^\pi\mu
\end{array}
\]
commutes for every return-distribution field $\mathcal Z$ with $F_{\mathcal Z}\in\mathcal X^{\mathrm{cdf}}$. The horizontal arrows represent the passage between representations: from a distribution field $\mathcal Z$ to its CDF field $F_{\mathcal Z}$, and from a CDF field to its spectral embedding via the transport $\mathbb V$.

This commutativity has two important consequences. First, all dynamical properties are preserved across the three representations. Contraction rates, fixed points, and convergence behaviour remain invariant, so one may analyse the Bellman update in whichever formulation is most convenient without altering the underlying reinforcement learning problem.

Second, this equivalence enables a clean separation of concerns. The probabilistic structure of the Bellman update is fully captured at the distribution and CDF levels, where the dynamics are most naturally defined. The spectral representation, meanwhile, provides a linear operator setting in which the induced action on centred CDF differences can be studied using Hilbert space methods. The transport $\mathbb V$ acts only as a change of coordinates and introduces no new dynamics or approximations.

This three-level equivalence is the foundation upon which the rest of our analysis builds. It allows us to establish all metric properties in the native CDF space $(\mathcal X^{\mathrm{cdf}},d_{\mathrm{Cr}})$ as we do in Section~\ref{subsec:cramer_contraction}, while subsequently leveraging the spectral representation for operator-theoretic investigations.
\end{remark}
\paragraph{Correspondence of fixed points.}
The conjugacy established above implies that fixed points of the Bellman update are preserved across all representation levels.

\begin{corollary}[CDF and spectral fixed points]
\label{cor:cdf_spec_fixed_point}
Let $\mathcal Z^\pi$ be the policy-induced return-distribution field satisfying $\mathcal T^\pi_{D}\mathcal Z^\pi=\mathcal Z^\pi$. Assume that $F_{\mathcal Z^\pi}\in\mathcal X^{\mathrm{cdf}}$, as required by Assumption~\ref{assump:gammaF_domain}. Then:
\begin{enumerate}
\item[\textup{(i)}] \textit{CDF fixed point.}\quad
$\mathcal T^\pi_{\mathrm{cdf}}\,F_{\mathcal Z^\pi}=F_{\mathcal Z^\pi}$.
\item[\textup{(ii)}] \textit{Spectral fixed point.}\quad
$\mu^{\pi}:=\mathbb V(F_{\mathcal Z^\pi})
\in\mathrm{range}(\mathbb V)$ satisfies
$\mathcal T^{\pi}\mu^{\pi}
=\mu^{\pi}$.
\end{enumerate}
\end{corollary}

\begin{proof}
Statement~(i) follows directly from Lemma~\ref{lemma:raw_cdf_intertwine}.
Statement~(ii) follows by applying Proposition~\ref{prop:intertwining_full} with $\mathcal Z=\mathcal Z^\pi$.
\end{proof}

\begin{remark}[Role of the CDF representation]
The results of this subsection show that the distributional, CDF-level, and spectral Bellman formulations are strictly equivalent. In particular, all dynamical and fixed-point properties of the Bellman update can be analysed entirely within the CDF space $(\mathcal X^{\mathrm{cdf}}, d_{\mathrm{Cr}})$. This observation allows us to focus exclusively on the Cram\'er geometry in the next subsection, where contraction and policy evaluation are established.
\end{remark}

\medskip
\noindent
The intertwining results established above show that no analytical content is gained or lost by passing between distributional, CDF-level, and spectral representations of the Bellman update. In particular, all stability, contraction, and fixed-point properties of the Bellman dynamics are already fully determined at the level of CDF fields endowed
with the Cram\'er geometry. We therefore return in the next subsection to a purely CDF-based analysis, where policy evaluation is established directly under the Cram\'er metric.
\subsection{Policy Evaluation under the Cram\'er Metric}
\label{subsec:cramer_contraction}
We now establish well-posedness of policy evaluation directly in the
Cram\'er geometry induced by CDF representations of return-distribution fields. By the intertwining results of Section~\ref{subsec:intertwining}, all dynamical and fixed-point properties of the distributional Bellman update are captured at the CDF level. It therefore suffices to carry out the contraction and fixed-point analysis on the metric space $(\mathcal X^{\mathrm{cdf}},d_{\mathrm{Cr}})$ introduced in Section~\ref{subsec:assumptions}, without further reference to spectral representations.

Recall that for admissible return-distribution fields $\mathcal Z_1,\mathcal Z_2$ with $F_{\mathcal Z_1},F_{\mathcal Z_2}\in\mathcal X^{\mathrm{cdf}}$, the state-action supremum Cram\'er metric is
\[
d_{\mathrm{Cr}}(F_{\mathcal Z_1},F_{\mathcal Z_2})
:=
\sup_{(s,a)\in\mathcal S\times\mathcal A}
\left(
\int_{\mathbb R}
\bigl(
F_{\mathcal Z_1(s,a)}(x)-F_{\mathcal Z_2(s,a)}(x)
\bigr)^2\,dx
\right)^{1/2}.
\]
For each fixed $(s,a)$, the integrand coincides with the classical Cram\'er distance between the corresponding return distributions.

We begin by bounding the individual CDF-level components of the Bellman update introduced in Section~\ref{subsec:bellman_unified_order}.

\begin{lemma}[Reward operator is Cram\'er non-expansive]
\label{lemma:cdf_reward_nonexp}
Let $\mathcal S_R$ be the reward translation operator defined in Section~\ref{subsec:bellman_unified_order}. Then, for all admissible return-distribution fields $\mathcal Z_1,\mathcal Z_2$ with $F_{\mathcal Z_1},F_{\mathcal Z_2}\in\mathcal X^{\mathrm{cdf}}$,
\[
d_{\mathrm{Cr}}(\mathcal S_R F_{\mathcal Z_1},
                 \mathcal S_R F_{\mathcal Z_2})
\;\le\;
d_{\mathrm{Cr}}(F_{\mathcal Z_1},F_{\mathcal Z_2}).
\]
\end{lemma}

\begin{proof}[Proof sketch]
The reward translation operator acts pointwise by averaging translations of CDFs. Since the $L^2(\mathbb R)$ norm is invariant under translation and convex under expectation, Jensen's inequality~\citep{jensen} implies that the Cram\'er distance between two CDF fields cannot increase under $\mathcal S_R$. A complete proof is provided in Appendix~\ref{appendix:proof_cdf_reward_nonexp}.
\end{proof}

\begin{lemma}[Discount operator is a $\sqrt{\gamma}$-contraction]
\label{lemma:cdf_discount_contraction}
Let $\mathcal D_\gamma$ be the discount scaling operator defined in Section~\ref{subsec:bellman_unified_order}.
Then, for all admissible return-distribution fields $\mathcal Z_1,\mathcal Z_2$ with $F_{\mathcal Z_1},F_{\mathcal Z_2}\in\mathcal X^{\mathrm{cdf}}$,
\[
d_{\mathrm{Cr}}(\mathcal D_\gamma F_{\mathcal Z_1},
                 \mathcal D_\gamma F_{\mathcal Z_2})
\;\le\;
\sqrt{\gamma}\,
d_{\mathrm{Cr}}(F_{\mathcal Z_1},F_{\mathcal Z_2}).
\]
\end{lemma}

\begin{proof}[Proof sketch]
The discount operator acts pointwise by dilation of the underlying CDF argument. A change of variables shows that this dilation rescales the $L^2(\mathbb R)$ distance between CDFs by a factor $\sqrt{\gamma}$. Taking the supremum over state-action pairs yields the stated $\sqrt{\gamma}$-contraction in the Cram\'er metric. A complete proof is provided in Appendix~\ref{appendix:proof_cdf_discount_contraction}.
\end{proof}

\begin{lemma}[Conditional expectation is Cram\'er non-expansive]
\label{lemma:cdf_condexp_nonexp}
Let $\mathcal C^\pi$ be as in Section~\ref{subsec:bellman_unified_order}. Then, for any return-distribution fields $\mathcal Z_1,\mathcal Z_2$ with $F_{\mathcal Z_1},F_{\mathcal Z_2}\in\mathcal X^{\mathrm{cdf}}$,
\[
d_{\mathrm{Cr}}\!\bigl(\mathcal C^\pi F_{\mathcal Z_1},\, \mathcal C^\pi F_{\mathcal Z_2}\bigr)
\;\le\;
d_{\mathrm{Cr}}\!\bigl(F_{\mathcal Z_1},\,F_{\mathcal Z_2}\bigr).
\]
\end{lemma}

\begin{proof}[Proof sketch]
The conditional expectation operator acts pointwise as a convex combination of successor CDFs. By Jensen's inequality, such averaging is non-expansive in the $L^2(\mathbb R)$ distance between CDFs, and taking the supremum over state-action pairs yields non-expansiveness in the Cram\'er metric. A complete proof is provided in Appendix~\ref{appendix:proof_cdf_condexp_nonexp}.
\end{proof}

Having established non-expansiveness and contraction properties for each primitive component of the Bellman update, we now combine these bounds to obtain a global contraction result for the CDF Bellman operator.
\begin{theorem}[Cram\'er contraction]
\label{thm:cramer_contraction}
Under Assumptions~\ref{assump:reward_bound}--\ref{assump:transition}, the CDF Bellman operator
\[
\mathcal T^\pi_{\mathrm{cdf}}
 = \mathcal C^\pi \circ \mathcal D_\gamma \circ \mathcal S_R
\]
is a $\sqrt{\gamma}$-contraction on the metric space $(\mathcal X^{\mathrm{cdf}}, d_{\mathrm{Cr}})$: for any return-distribution fields $\mathcal Z_1,\mathcal Z_2$ with $F_{\mathcal Z_1},F_{\mathcal Z_2}\in\mathcal X^{\mathrm{cdf}}$,
\[
d_{\mathrm{Cr}}\!\bigl(
\mathcal T^\pi_{\mathrm{cdf}} F_{\mathcal Z_1},
\mathcal T^\pi_{\mathrm{cdf}} F_{\mathcal Z_2}
\bigr)
\le
\sqrt{\gamma}\,
d_{\mathrm{Cr}}\!\bigl(F_{\mathcal Z_1},F_{\mathcal Z_2}\bigr).
\]
\end{theorem}

\begin{proof}
Let $\mathcal Z_1,\mathcal Z_2$ be arbitrary return-distribution fields such that $F_{\mathcal Z_1},F_{\mathcal Z_2}\in\mathcal X^{\mathrm{cdf}}$. By Lemmas~\ref{lemma:cdf_reward_nonexp}, \ref{lemma:cdf_discount_contraction}, and~\ref{lemma:cdf_condexp_nonexp}, we have
\begin{align*}
d_{\mathrm{Cr}}\!\bigl(
\mathcal T^\pi_{\mathrm{cdf}} F_{\mathcal Z_1},
\mathcal T^\pi_{\mathrm{cdf}} F_{\mathcal Z_2}
\bigr)
&=
d_{\mathrm{Cr}}\bigl(
\mathcal C^\pi \mathcal D_\gamma \mathcal S_R F_{\mathcal Z_1},\,
\mathcal C^\pi \mathcal D_\gamma \mathcal S_R F_{\mathcal Z_2}
\bigr) \\
&\le
d_{\mathrm{Cr}}\bigl(
\mathcal D_\gamma \mathcal S_R F_{\mathcal Z_1},\,
\mathcal D_\gamma \mathcal S_R F_{\mathcal Z_2}
\bigr) \\
&\le
\sqrt{\gamma}\,
d_{\mathrm{Cr}}\bigl(
\mathcal S_R F_{\mathcal Z_1},\,
\mathcal S_R F_{\mathcal Z_2}
\bigr) \\
&\le
\sqrt{\gamma}\,
d_{\mathrm{Cr}}\!\bigl(F_{\mathcal Z_1},F_{\mathcal Z_2}\bigr),
\end{align*}
which is the desired contraction bound.
\end{proof}
To conclude policy evaluation via a fixed-point argument, it remains to verify that the underlying metric space is complete. We therefore establish completeness of the admissible Cram\'er CDF space.

\begin{lemma}[Completeness~\citep{completeness} of the Cram\'er CDF space]
\label{lem:P2_complete}
The metric space $(\Gamma_F,d_C)$ is complete.
\end{lemma}
\begin{proof}[Proof sketch]
The Cram\'er metric $d_C$ coincides with the $L^2(\mathbb R)$ distance between CDFs. Since $L^2(\mathbb R)$ is complete, any Cauchy sequence in $(\Gamma_F,d_C)$ admits an $L^2$ limit, which can be shown to admit a right-continuous, nondecreasing representative with the correct boundary behaviour at $\pm\infty$. A complete proof is provided in Appendix~\ref{appendix:proof_cramer_cdf_complete}.
\end{proof}

Together with the contraction property of Theorem~\ref{thm:cramer_contraction}, completeness of the Cram\'er CDF space ensures that the Banach fixed-point theorem~\citep{banach} applies.

\begin{corollary}[Fixed point existence and uniqueness in CDF space]
\label{cor:cdf_fixed_point}
Under Assumption~\ref{assump:gammaF_domain}, the CDF Bellman operator $\mathcal T^\pi_{\mathrm{cdf}}$ admits a unique fixed point $F_{\mathcal Z^\pi}\in\mathcal X^{\mathrm{cdf}}$. Moreover, for any initial return-distribution field $\mathcal Z_0$ with $F_{\mathcal Z_0}\in\mathcal X^{\mathrm{cdf}}$, the iterates
\[
F_{\mathcal Z_{n+1}} = \mathcal T^\pi_{\mathrm{cdf}} F_{\mathcal Z_n},
\qquad n\ge0,
\]
converge to $F_{\mathcal Z^\pi}$ in $(\mathcal X^{\mathrm{cdf}},d_{\mathrm{Cr}})$:
\[
d_{\mathrm{Cr}}(F_{\mathcal Z_n}, F_{\mathcal Z^\pi})
\xrightarrow[]{} 0 \qquad \text{ as $n\to\infty$}.
\]
\end{corollary}

\begin{proof}[Proof sketch]
By Theorem~\ref{thm:cramer_contraction}, the operator $\mathcal T^\pi_{\mathrm{cdf}}$ is a strict contraction on $(\mathcal X^{\mathrm{cdf}}, d_{\mathrm{Cr}})$.

Completeness of $(\mathcal X^{\mathrm{cdf}}, d_{\mathrm{Cr}})$ follows from the completeness of $(\Gamma_F, d_C)$ (Lemma~\ref{lem:P2_complete}) by a standard pointwise argument on state-action pairs.

The claim then follows directly from the Banach fixed-point theorem. A complete proof is provided in Appendix~\ref{appendix:proof_cdf_fixed_point}.
\end{proof}

\begin{remark}[Invariance and well-posedness]
\label{remark:cdf_invariance}
Under Assumptions~\ref{assump:reward_bound} and~\ref{assump:transition}, the admissible CDF class $\Gamma_F$ is invariant under the CDF Bellman operator $\mathcal T^\pi_{\mathrm{cdf}}$. Together with completeness of the Cram\'er geometry, this yields existence and uniqueness of a fixed point.

Under the additional Assumption~\ref{assump:gammaF_domain}, this fixed point corresponds to the policy-induced return distribution, and distributional policy evaluation is well posed.
\end{remark}

Having established existence, uniqueness, and convergence of the Bellman fixed point at the CDF level, we now transport this result into the regularised spectral Hilbert space using the CDF-spectral correspondence developed in Section~\ref{subsec:cdf_spectral_bridge}.
\subsection{Fixed Point Correspondence in Spectral Space}
\label{subsec:spectral_fixed_point}
Throughout this subsection, all Hilbert spaces are understood over the real field. Accordingly, the maps $\mathcal U$, $\mathcal V$, and $\mathbb V$ are linear operators between real Hilbert spaces.

Having established existence, uniqueness, and convergence of the Bellman fixed point in the CDF space $(\mathcal X^{\mathrm{cdf}}, d_{\mathrm{Cr}})$ in Section~\ref{subsec:cramer_contraction}, we now study how these results are transported to the spectral representation. Our goal is not to introduce an independent spectral analysis, but to make precise how well-posedness of policy evaluation is inherited under the conjugate transport induced by $\mathbb V$.

\paragraph{Admissible spectral field domain.}
To avoid introducing extraneous spectral objects unrelated to the original distributional problem, we restrict attention to the spectral image of admissible CDF fields.
Recall from Section~\ref{subsec:intertwining} that the lifted transport
\[
\mathbb V:\mathcal X^{\mathrm{cdf}}\longrightarrow \mathcal X^{\mathrm{spec}}
\]
is defined pointwise by $(\mathbb V F_{\mathcal Z})(s,a) := \mathcal V\!\bigl(F_{\mathcal Z(s,a)}\bigr)$, where $\mathcal V$ denotes the raw-to-spectral transport on individual CDFs. We define the \emph{admissible spectral field domain} through
\begin{equation}
\mathcal X^{\mathrm{adm}}
\;:=\;
\mathrm{range}(\mathbb V)
\;\subset\;
\mathcal X^{\mathrm{spec}} .
\label{eq:X_adm_def}
\end{equation}
By construction, $\mathbb V$ restricts to a bijection between $\mathcal X^{\mathrm{cdf}}$ and $\mathcal X^{\mathrm{adm}}$, with inverse $\mathbb V^{-1}$ defined pointwise via $\mathcal V^{-1}$.

\paragraph{Induced spectral metric.}
Since the lifted transport $\mathbb V$ acts pointwise on state-action pairs through the raw-to-spectral map $\mathcal V$, the Cram\'er geometry on return distributions can be transported to the spectral representation in a pointwise manner. We first introduce the induced Cram\'er distance at the level of individual spectral embeddings, and then aggregate it to obtain a metric on spectral fields.

\begin{definition}[Pointwise induced Cram\'er distance]
\label{def:pointwise_induced_metric}
For $\mu_1,\mu_2\in\mathcal X^{\mathrm{adm}}$ and $(s,a)\in\mathcal S\times\mathcal A$,
define
\begin{equation}
\tilde d_C\!\bigl(\mu_1(s,a),\mu_2(s,a)\bigr)
\;:=\;
d_C\!\left(
\mathbb V^{-1}\mu_1(s,a),\,
\mathbb V^{-1}\mu_2(s,a)
\right),
\label{eq:pointwise_tilde_dc}
\end{equation}
where $d_C$ denotes the Cram\'er distance.  We slightly overload the notation and write $d_C(F_{P_1},F_{P_2})$ when  $F_{P_i}$ are the CDFs of probability measures $P_i$, identifying them with the corresponding laws.
\end{definition}

Using the pointwise distance $\tilde d_C$, we equip the admissible spectral field domain with the induced Cram\'er metric defined by aggregation over state-action pairs.

\begin{definition}[Induced spectral Cram\'er metric]
\label{def:induced_metric}
For $\mu_1,\mu_2\in\mathcal X^{\mathrm{adm}}$, define
\begin{equation}
\tilde d_{\mathrm{Cr}}(\mu_1,\mu_2)
\;:=\;
\sup_{(s,a)\in\mathcal S\times\mathcal A}
\tilde d_C\!\bigl(\mu_1(s,a),\mu_2(s,a)\bigr)\;=\;
d_{\mathrm{Cr}}\!\left(
\mathbb V^{-1}\mu_1,\,
\mathbb V^{-1}\mu_2
\right).
\label{eq:induced_metric}
\end{equation}
\end{definition}

By construction, the mapping
\[
\mathbb V:(\mathcal X^{\mathrm{cdf}},d_{\mathrm{Cr}})
\longrightarrow
(\mathcal X^{\mathrm{adm}},\tilde d_{\mathrm{Cr}})
\]
is an isometric bijection. In particular, $\tilde d_C$ represents the pointwise pullback of the Cram\'er distance under the spectral embedding, while $\tilde d_{\mathrm{Cr}}$ coincides with its field-level aggregation.

\begin{proposition}[Cram\'er contraction under the induced spectral metric]
\label{prop:spectral_contraction}
Under Assumptions~\ref{assump:reward_bound}--\ref{assump:transition}, the spectral Bellman operator $\mathcal T^{\pi} \;:=\; \mathbb V \circ \mathcal T^\pi_{\mathrm{cdf}} \circ \mathbb V^{-1}$ is a $\sqrt{\gamma}$-contraction on $\bigl(\mathcal X^{\mathrm{adm}},\tilde d_{\mathrm{Cr}}\bigr)$: for any $\mu_1,\mu_2\in\mathcal X^{\mathrm{adm}}$,
\begin{equation}
\label{eq:spectral_contraction}
\tilde d_{\mathrm{Cr}}\!\bigl(\mathcal T^{\pi}\mu_1,\mathcal T^{\pi}\mu_2\bigr)
\;\le\;
\sqrt{\gamma}\,
\tilde d_{\mathrm{Cr}}(\mu_1,\mu_2).
\end{equation}
\end{proposition}

\begin{proof}
Let $\mu_1,\mu_2\in\mathcal X^{\mathrm{adm}}$ and set $F_i:=\mathbb V^{-1}\mu_i\in\mathcal X^{\mathrm{cdf}}$. Using the definition of $\tilde d_{\mathrm{Cr}}$ in~\eqref{eq:induced_metric} and the conjugacy definition~\eqref{eq:Tpi_spec_function}, we obtain
\begin{align*}
\tilde d_{\mathrm{Cr}}\!\bigl(\mathcal T^{\pi}\mu_1,\mathcal T^{\pi}\mu_2\bigr)
&=
d_{\mathrm{Cr}}\!\left(
\mathbb V^{-1}\mathcal T^{\pi}\mu_1,\,
\mathbb V^{-1}\mathcal T^{\pi}\mu_2
\right) \\
&=
d_{\mathrm{Cr}}\!\left(
\mathcal T^\pi_{\mathrm{cdf}} F_1,\,
\mathcal T^\pi_{\mathrm{cdf}} F_2
\right).
\end{align*}
By Theorem~\ref{thm:cramer_contraction},
$\mathcal T^\pi_{\mathrm{cdf}}$ is a $\sqrt{\gamma}$-contraction on
$\bigl(\mathcal X^{\mathrm{cdf}},d_{\mathrm{Cr}}\bigr)$, hence
\[
d_{\mathrm{Cr}}\!\left(
\mathcal T^\pi_{\mathrm{cdf}} F_1,\,
\mathcal T^\pi_{\mathrm{cdf}} F_2
\right)
\le
\sqrt{\gamma}\,
d_{\mathrm{Cr}}(F_1,F_2).
\]
Finally, $d_{\mathrm{Cr}}(F_1,F_2)=\tilde d_{\mathrm{Cr}}(\mu_1,\mu_2)$ by \eqref{eq:induced_metric}, which gives~\eqref{eq:spectral_contraction}.
\end{proof}

\paragraph{Completeness of the spectral representation.}
Having established that the conjugate spectral Bellman operator is contractive under the induced metric, we next verify that $\bigl(\mathcal X^{\mathrm{adm}},\tilde d_{\mathrm{Cr}}\bigr)$ is complete, so that the transported dynamics is well posed on a closed metric space.

\begin{proposition}[Completeness of the admissible spectral space]
\label{prop:spectral_completeness}
The metric space $(\mathcal X^{\mathrm{adm}},\tilde d_{\mathrm{Cr}})$ is complete.
\end{proposition}

\begin{proof}
Let $\{\mu_n\}_{n\ge0}$ be a Cauchy sequence in $(\mathcal X^{\mathrm{adm}},\tilde d_{\mathrm{Cr}})$. By definition of the induced metric, $\{\mathbb V^{-1}\mu_n\}_{n\ge0}$ is a Cauchy sequence in $(\mathcal X^{\mathrm{cdf}},d_{\mathrm{Cr}})$, which is complete by Corollary~\ref{cor:cdf_fixed_point}. Then let $G\in\mathcal X^{\mathrm{cdf}}$ denote its limit. Since $\mathbb V$ is bijective onto the admissible domain, $\mu_n\to\mathbb V(G)\in\mathcal X^{\mathrm{adm}}$, establishing completeness.
\end{proof}
This result ensures that the admissible spectral space provides a closed geometric setting for dynamics transported from the CDF level.

\begin{remark}[Role of completeness in the spectral representation]
The completeness result above is not used to derive an independent policy evaluation procedure in the spectral space, nor does it introduce new dynamics. Rather, it establishes that the admissible spectral image $\mathcal X^{\mathrm{adm}}$, when equipped with the pullback Cram\'er metric $\tilde d_{\mathrm{Cr}}$, forms a closed and well-defined geometric setting.

In particular, limits and fixed points obtained by conjugacy from the CDF-level Bellman dynamics are guaranteed to remain within $\mathcal X^{\mathrm{adm}}$. This ensures that the spectral representation can host the geometric structures induced by the Cram\'er metric, without introducing new dynamics or additional analytical assumptions.
\end{remark}

\paragraph{Spectral Bellman fixed point.}
We now identify the fixed point of the spectral Bellman operator. Rather than deriving this result from an independent contraction argument in the spectral space, the fixed point is obtained directly by conjugacy from the CDF-level Bellman dynamics, reflecting the fact that all analytical content originates at the CDF level.

\begin{theorem}[Well-posedness of the spectral Bellman fixed point]
\label{thm:spectral_fixed_point}
Assume Assumptions~\ref{assump:reward_bound}--\ref{assump:gammaF_domain}. Then for every $\epsilon>0$, the spectral Bellman operator $\mathcal T^{\pi}$ admits a unique fixed point $\mu^{\pi}\in\mathcal X^{\mathrm{adm}}$, given by
\begin{equation}
\mu^{\pi} = \mathbb V\!\left(F_{\mathcal Z^\pi}\right).
\label{eq:spectral_fixed_point_expression}
\end{equation}
Equivalently, for each $(s,a)\in\mathcal S\times\mathcal A$,
\begin{equation}
\mu^{\pi}(s,a)
= \mathcal V\!\left(F_{\mathcal Z^\pi(s,a)}\right)
= \Phi\!\left(\mathcal Z^\pi(s,a)\right),
\label{eq:spectral_fixed_point_pointwise}
\end{equation}
where $\Phi(P)=\mathcal V(F_P)$ denotes the regularised spectral embedding of a probability measure $P$.
\end{theorem}

\begin{proof}
By Corollary~\ref{cor:cdf_fixed_point}, the CDF Bellman operator $\mathcal T^\pi_{\mathrm{cdf}}$ admits a unique fixed point in $\mathcal X^{\mathrm{cdf}}$, namely the CDF field $F_{\mathcal Z^\pi}$ of the policy-induced return-distribution field. Define
\[
\mu^{\pi}:=\mathbb V\!\left(F_{\mathcal Z^\pi}\right)\in\mathcal X^{\mathrm{adm}}.
\]
Using the conjugation definition~\eqref{eq:Tpi_spec_function}, we obtain
\[
\mathcal T^{\pi}\mu^{\pi}
=
\mathbb V\!\left(
\mathcal T^\pi_{\mathrm{cdf}}\,F_{\mathcal Z^\pi}
\right)
=
\mathbb V\!\left(F_{\mathcal Z^\pi}\right)
=
\mu^{\pi},
\]
so $\mu^{\pi}$ is a fixed point. Uniqueness follows by applying $\mathbb V^{-1}$ and invoking uniqueness of the CDF fixed point.
\end{proof}

The result above establishes that, for each fixed regularisation parameter $\epsilon>0$, the spectral Bellman operator admits a unique fixed point $\mu^{\pi}\in\mathcal X^{\mathrm{adm}}$, obtained by conjugacy from the CDF-level Bellman dynamics.

While this identifies a well-defined spectral representative of the policy-induced return distributions at every finite regularisation scale, it leaves open the question of how these Hilbert space realisations relate to the intrinsic Cram\'er geometry in the limit $\epsilon\downarrow0$. We address this question in the next subsection by studying convergence at the level of induced metrics.

\subsection{Recovery of the Cram\'er Geometry}
\label{subsec:strong_limit_eps}
Having established the existence and uniqueness of the spectral Bellman fixed point for each fixed regularisation parameter $\epsilon>0$ (Section~\ref{subsec:spectral_fixed_point}), we now clarify how these regularised Hilbert space realisations relate to the intrinsic Cram\'er geometry of return distributions as $\epsilon$ decreases.

\paragraph{Geometry versus representation.}
The central point is that the parameter $\epsilon$ controls a \emph{regularised representation} of the Cram\'er geometry rather than defining a family of approximations within a single ambient Hilbert space. Indeed, the Hilbert spaces $\mathcal H_\epsilon$ vary with $\epsilon$, and there is no fixed Hilbert structure with respect to which the spectral fixed points $\mu^\pi$ converge as vectors.

What is preserved and refined as $\epsilon\downarrow0$ is not vectorial convergence, but the \emph{distance structure} induced by these representations. In other words, convergence takes place at the level of metrics, rather than within any single normed space.

\paragraph{Regularised induced distances.}
Recall that the raw-to-spectral map $\mathcal V:\Gamma_F\to\mathcal H_\epsilon$ induces a regularised Cram\'er distance on admissible return distributions via
\begin{equation}
d_{C,\epsilon}(P_1,P_2) \;:=\;
\bigl\|\mathcal V(F_{P_1})-\mathcal V(F_{P_2})\bigr\|_{\mathcal H_\epsilon}, \qquad \epsilon>0.
\label{eq:dc_eps_recall}
\end{equation}
Similarly, here we also overload the notation and write $d_{C,\epsilon}(F_{P_1},F_{P_2})$ when $F_{P_i}$ are CDFs, identifying them with the corresponding probability measures.
This distance coincides with the Hilbert norm in $\mathcal H_\epsilon$, but depends explicitly on the regularisation parameter.

At the level of spectral fields, this induces a pointwise distance through the inverse lifted transport. For $\mu_1,\mu_2\in\mathcal X^{\mathrm{adm}}$ and $(s,a)\in\mathcal S\times\mathcal A$, we can say
\begin{equation}
\tilde d_{C,\epsilon}\!\bigl(\mu_1(s,a),\mu_2(s,a)\bigr) \;:=\;
d_{C,\epsilon}\!\left( \mathbb V^{-1}\mu_1(s,a),\,
\mathbb V^{-1}\mu_2(s,a) \right).
\label{eq:tilde_dc_eps_recall}
\end{equation}
By construction of $\mathcal X^{\mathrm{adm}}=\mathrm{range}(\mathbb V)$, this distance admits the explicit Hilbertian realisation
\begin{equation}
\tilde d_{C,\epsilon}\!\bigl(\mu_1(s,a),\mu_2(s,a)\bigr)
= \bigl\|\mu_1(s,a)-\mu_2(s,a)\bigr\|_{\mathcal H_\epsilon},
\qquad \epsilon>0.
\label{eq:tilde_dc_eps_equals_norm_recall}
\end{equation}
Aggregating over state-action pairs yields the regularised induced spectral metric
\begin{equation}
\tilde d_{\mathrm{Cr},\epsilon}(\mu_1,\mu_2)
\;:=\; \sup_{(s,a)\in\mathcal S\times\mathcal A}
\tilde d_{C,\epsilon}\!\bigl(\mu_1(s,a),\mu_2(s,a)\bigr)
= \|\mu_1-\mu_2\|_{\infty,\epsilon},
\label{eq:tilde_dcr_eps_equals_supnorm}
\end{equation}
where $\|\cdot\|_{\infty,\epsilon}$ is defined in~\eqref{eq:X_spec_norm}. Thus, for each fixed $\epsilon>0$, the Hilbert norm on $\mathcal H_\epsilon$ provides a concrete realisation of a \emph{regularised Cram\'er geometry} on admissible spectral fields.

\paragraph{Monotone recovery of the Cram\'er geometry.}
The relationship between the regularised distances and the intrinsic Cram\'er geometry is governed by the monotone convergence established in Proposition~\ref{prop:cramer_spectral_bridge}. For any admissible probability measures $P_1,P_2$,
\begin{equation}
\bigl\|\mathcal V(F_{P_1})-\mathcal V(F_{P_2})\bigr\|_{\mathcal H_\epsilon} \;\uparrow\; d_C(P_1,P_2)
\qquad \text{as }\epsilon\downarrow0.
\label{eq:dc_eps_to_dc}
\end{equation}

This convergence holds pointwise for each pair of distributions and reflects the fact that the regularised spectral norms approximate the Cram\'er geometry from below.

Because the lifted transport $\mathbb V$ acts pointwise on state-action pairs, this convergence extends naturally to the spectral field level. In particular, since the field-level metrics are defined by taking the supremum over $(s,a)$ and satisfy $d_{C,\epsilon}\le d_C$ for all $\epsilon>0$, the monotone convergence in~\eqref{eq:dc_eps_to_dc} yields
\begin{equation}
\tilde d_{\mathrm{Cr},\epsilon}(\mu_1,\mu_2)
\;\uparrow\; \tilde d_{\mathrm{Cr}}(\mu_1,\mu_2)
\qquad \text{as }\epsilon\downarrow0,
\label{eq:dcr_eps_to_dcr}
\end{equation}
for all $\mu_1,\mu_2\in\mathcal X^{\mathrm{adm}}$, where $\tilde d_{\mathrm{Cr}}$ denotes the pullback of the intrinsic Cram\'er metric defined in Section~\ref{subsec:spectral_fixed_point}.

\begin{theorem}[Recovery of the Cram\'er geometry]
\label{thm:recovery_cramer_geometry}
For every $(s,a)\in\mathcal S\times\mathcal A$ and every probability measure $P$ with $F_P\in\Gamma_F$,
\begin{equation}
\bigl\|\mu^{\pi}(s,a)-\mathcal V(F_P)\bigr\|_{\mathcal H_\epsilon}
\;\uparrow\; d_C\!\bigl(\mathcal Z^\pi(s,a),P\bigr)
\qquad \text{as }\epsilon\downarrow0 .
\label{eq:recovery_pointwise}
\end{equation}
Equivalently, the family of regularised induced metrics $\{\tilde d_{\mathrm{Cr},\epsilon}\}_{\epsilon>0}$ converges monotonically to the intrinsic Cram\'er geometry on admissible return distributions.
\end{theorem}

\begin{proof}
By definition of the spectral fixed point, $\mu^\pi(s,a)=\mathcal V(F_{\mathcal Z^\pi(s,a)})$. Substituting into~\eqref{eq:recovery_pointwise} and applying Proposition~\ref{prop:cramer_spectral_bridge} yields the claim.
\end{proof}

\paragraph{Interpretation and implications.}
Theorem~\ref{thm:recovery_cramer_geometry} clarifies the precise role of the regularisation parameter $\epsilon$. Although for each fixed $\epsilon>0$ the spectral fixed point $\mu^\pi(s,a)$ lives in a genuine Hilbert space $\mathcal H_\epsilon$, the zero-regularisation limit does not produce an element in a limiting Hilbert space. What emerges instead is a metric structure: the intrinsic Cram\'er geometry of admissible return distributions.

In this sense, the Cram\'er metric can be understood as a singular limit of a family of regularised Hilbert geometries. The low-frequency obstruction of the unregularised spectral representation is resolved at every $\epsilon>0$, yielding a stable inner-product structure in which the induced linear action of the Bellman update on centred CDF representations admits a bounded operator realisation. As $\epsilon\downarrow0$, the induced distances associated with these Hilbert structures converge monotonically to the native Cram\'er metric.

From a practical perspective, this interpretation suggests two important implications. First, for sufficiently small but nonzero $\epsilon$, policy evaluation performed in the regularised spectral representation provides a controlled Hilbertian approximation of Cram\'er-geometry policy evaluation. Second, the regularised setting offers a well-defined linear and inner-product framework at finite $\epsilon$, potentially enabling spectral approximation schemes, kernel-based constructions, or parametric representations that are not directly available in the singular Cram\'er geometry.

Taken together, Sections~\ref{subsec:cdf_spectral_bridge}--\ref{subsec:strong_limit_eps} therefore establish more than a formal equivalence of representations. They identify the Cram\'er geometry as the zero-regularisation limit of a family of stable Hilbert realisations of the distributional Bellman dynamics, clarifying both the analytic structure of the metric and the scope of its regularised spectral approximations.

\section{Related Work}
\label{sec:relatedwork}
Distributional reinforcement learning studies the evolution of full return distributions under Bellman updates~\citep{c51,drlbook_bellemare_distributional_2023,drllinear,drlstat}, rather than focusing solely on expected returns. For fixed policies, a substantial body of work has established that the distributional Bellman operator admits contraction and fixed-point properties under certain metrics~\citep{drlbook_bellemare_distributional_2023,c51theory}. While analogous guarantees do not generally extend to the corresponding optimality operator~\citep{drlbook_bellemare_distributional_2023,c51theory}, the present work deliberately focuses on the policy-evaluation setting, where the Bellman operator constitutes a well-defined and analytically tractable object. Our analysis is therefore concerned with the structural properties of the policy-evaluation Bellman operator itself, independent of control-specific considerations, and can be viewed as a foundational step toward understanding more general settings.

\paragraph{Algorithmic representations and approximation schemes.}
A wide range of algorithmic approaches to distributional reinforcement learning have been proposed, based on different parameterisations of return distributions, including categorical~\citep{c51,d4pg}, quantile-based~\citep{qrdqn-Dabney18,iqn,fqf,dsacma,dpo}, kernel-based~\citep{mmddrl}, and parametric~\citep{dsacDuan_2022,dsactDuan_2025} representations. While these methods differ in representation, projection, and loss construction~\citep{lossdrl}, they are typically analysed as approximations of the same underlying distributional Bellman dynamics in order to propose algorithms. Theoretical studies in this line primarily address the effects of finite representations, projection operators, and stochastic approximation on stability and convergence~\citep{splinedrl,sketchdrl,lossdrl,mmddrl2}. In contrast, the present paper does not introduce a new algorithmic representation, nor does it study approximation effects. Instead, it focuses on the exact, unprojected distributional Bellman operator, independently of any particular parameterisation.

\paragraph{Metrics on return distributions and the Cram\'er geometry.}
The behaviour of distributional Bellman updates depends critically on the geometry imposed on the space of return distributions~\citep{c51,c51theory}. Wasserstein-type metrics play an important role in early analyses, particularly in establishing contraction properties of distributional Bellman operators, however, their non-Hilbertian nature makes it less straightforward to develop operator-theoretic or spectral representations of Bellman dynamics within this geometry~\citep{drlbook_bellemare_distributional_2023,wasserstein1,wasser}. An alternative is provided by CDF-based metrics, among which the $L^2$ distance between cumulative distribution functions~\citep{returncdf}, commonly referred to as the Cram\'er metric, occupies a distinguished position~\citep{c51theory,drlstat}. This geometry underlies categorical distributional methods and several theoretical treatments of distributional Bellman operators in the policy-evaluation setting. Existing work in this direction is largely metric in nature~\citep{ddp} and is primarily concerned with contraction, fixed-point existence, and convergence.

\paragraph{Spectral and functional perspectives.}
Spectral representations based on characteristic functions and Fourier transforms are classical tools~\citep{cf,cfthm} in probability theory for analysing distributions and their transformations, such as translation, scaling, and averaging.

However, these spectral tools have not been developed into a Hilbert space framework for analysing distributional Bellman operators. In particular, existing work does not provide a formulation that is simultaneously aligned with the CDF-based Cram\'er geometry and exactly compatible with the unmodified distributional Bellman update. As a result, prior analyses of Cram\'er-based Bellman dynamics remain primarily metric in nature.

\paragraph{Positioning of the present work.}
The present paper is complementary to the above lines of research. Rather than proposing a new distributional algorithm or studying approximation effects, it addresses a structural question: how Cram\'er-based distributional Bellman dynamics can be embedded within a Hilbertian analytical framework without altering the underlying update. The motivation for this construction is conceptual and analytical. While the Cram\'er contraction principle guarantees well-posedness at the metric level, it does not by itself yield a representation in which the induced difference-level structure of the Bellman dynamics becomes accessible to operator-theoretic analysis. By constructing a Hilbertian envelope of the Cram\'er geometry, we provide such a representation while remaining exactly consistent with the intrinsic CDF-level formulation. This focus distinguishes the present work from existing metric-based analyses and lays groundwork for subsequent developments grounded in Hilbert space methods.

\section{Discussion}
\label{sec:discussion}
This paper forms part of a broader analytical effort to understand distributional reinforcement learning under the Cram\'er geometry. Its specific contribution is to identify and resolve a geometric obstruction that arises when one attempts to analyse distributional Bellman updates using linear and Hilbert space-based tools, without modifying the underlying Bellman dynamics or their probabilistic interpretation.

\subsection{The geometric obstruction in Cram\'er-based Bellman analysis}
Distributional Bellman dynamics admit a natural and well-posed formulation at the level of CDFs, where contraction, fixed-point existence, and convergence follow directly from the Cram\'er geometry. As shown in this work, policy evaluation under a fixed policy is fully characterised within the metric space $(\mathcal X^{\mathrm{cdf}}, d_{\mathrm{Cr}})$, without reference to any spectral or Hilbert space representation.

From a functional-analytic perspective, however, the Cram\'er geometry exhibits a structural limitation. Although the Cram\'er metric coincides with an $L^2$ distance between CDFs, its Fourier representation involves a singular low-frequency weight~\citep{drlbook_bellemare_distributional_2023,cramer-singularDudley_2002}. This singularity is benign for metric arguments but prevents the induced linear action of the Bellman update on centred CDF differences from being realised as a bounded operator in a spectral Hilbert space. As a result, existing Cram\'er-based analyses have largely remained at the CDF or metric level~\citep{c51theory,drllinear,drlstat}. In this sense, the limitation of existing analyses is not the absence of a contraction principle, but the absence of a Hilbertian operator framework that internalises the difference-level linear structure of the Bellman dynamics.

\subsection{Resolving the obstruction via geometric regularisation}
The central contribution of this paper is to resolve the above obstruction by regularising the geometry used to analyse the Cram\'er metric, without modifying the underlying Bellman dynamics. Rather than altering the Bellman update or introducing approximations at the level of return distributions, we construct a family of Hilbert spaces that form a Hilbertian envelope of the Cram\'er geometry.

For each fixed regularisation parameter $\epsilon>0$, this construction yields a Hilbert space setting in which the induced linear action of the CDF-level Bellman update on centred representations admits an exact conjugate operator realisation via an isometric correspondence. All contraction and fixed-point properties are established intrinsically at the CDF level, and the Hilbert space formulation is introduced only afterwards as a representational tool.

This separation between intrinsic dynamics and analytical representation is precisely what enables the use of linear functional tools, such as spectral decompositions, perturbation arguments, or sensitivity analyses, without altering the underlying Bellman operator.

The role of the parameter $\epsilon$ is purely geometric. It does not approximate the Bellman operator, smooth the return distributions, or modify the reinforcement learning problem itself. Instead, $\epsilon$ attenuates the low-frequency singularity present in the Fourier representation of the Cram\'er metric. As $\epsilon \downarrow 0$, the induced Hilbert metrics converge monotonically to the intrinsic Cram\'er distance on its finite domain, ensuring consistency with the original CDF-level analysis.

\subsection{Scope and outlook}
The analysis in this paper is deliberately restricted to policy evaluation under a fixed policy and does not address control, function approximation, or algorithmic implementations. These restrictions are intentional and reflect the aim of isolating the geometric structure induced by the Cram\'er metric before introducing additional sources of approximation or complexity.

By establishing a Hilbertian framework that is compatible with CDF-level Bellman dynamics, this work provides a foundation for further analysis based on linear representations and continuous functionals, while retaining exact correspondence with the original distributional formulation. In particular, it prepares the setting for subsequent investigations using results from Hilbert space theory, such as the Riesz representation theorem~\citep{operator-conjugate} or perturbative operator analysis, which require precisely the type of geometric structure developed here. This structure not only enables theoretical investigations, such as spectral and perturbation analysis, but also provides a principled starting point for developing approximation schemes with rigorous error guarantees, a crucial step toward bridging foundational theory and practical algorithm design.

\section{Conclusion}
\label{sec:conclusion}
This paper revisits the distributional Bellman operator under the Cram\'er metric from a geometric and representational perspective. While the Cram\'er contraction property guarantees well-posedness of distributional policy evaluation, it does not by itself furnish a functional-analytic setting in which the induced linear action of the Bellman update on centred CDF differences can be realised as a bounded operator on a Hilbert space. The central difficulty lies not in the Bellman dynamics themselves, which remain affine and well defined at the CDF level, but in the mismatch between the intrinsic Cram\'er geometry and the requirements of a stable Hilbert space representation.

Our analysis separates these aspects. All probabilistic, metric, and fixed-point properties of the distributional Bellman update are formulated and established directly at the level of CDFs, where the Cram\'er geometry is native and complete. Building on this intrinsic formulation, we construct a family of spectrally regularised Hilbert spaces that realise the same geometry via exact transport, without modifying the Bellman update or introducing auxiliary approximations. The resulting framework provides a Hilbertian representational envelope that internalises the induced linear structure of the Bellman update on centred CDF representations, while recovering the intrinsic Cram\'er geometry in the zero-regularisation limit.

The scope of the present work is deliberately restricted to policy evaluation under a fixed policy. We do not consider optimality operators, function approximation, or algorithmic implementations. These limitations are intentional: the goal is to isolate the geometric and representational structure induced by the Cram\'er metric before introducing additional sources of approximation or control.

Viewed in this light, the contribution of this paper is foundational. By clarifying the relationship between CDF-level Bellman dynamics and their regularised Hilbert space realisations, it establishes a principled operator-theoretic framework in which the induced difference-level structure of the Bellman update can be analysed using Hilbert space methods. This framework has direct implications for practice: it provides the theoretical tools needed to understand why and when distributional RL algorithms work, to diagnose instabilities in neural network function approximation, and to design loss functions and architectures that respect the underlying geometry of return distributions. For example, the spectral perspective opens the door to principled regularisation schemes that control approximation error propagation, or to perturbation analyses that quantify how model misspecification affects learned policies. These practical questions, while central to the empirical success of distributional RL, have previously lacked a rigorous functional-analytic foundation due to the singular spectral structure of the raw Cram\'er metric. By resolving this representational obstruction, the present work bridges the gap between the geometric theory of distributional Bellman operators and the algorithmic design choices that determine performance in practice. Further developments along these lines, including concrete algorithmic instantiations, are left to future work.


\acks{This publication has emanated from research supported in part by a grant from Taighde Éireann - Research Ireland under Grant number 18/CRT/6049. For the purpose of Open Access, the author has applied a CC BY public copyright licence to any Author Accepted Manuscript version arising from this submission.}


\newpage

\appendix
\section{Proof Details in Section~\ref{sec:hilbert}}

\subsection{Proof of Lemma~\ref{lemma:fourier_cdf_difference}}
\label{appendix:proof_fourier_cdf_difference}

\begin{proof}
Given $P_1,P_2$ with CDFs $F_{P_1},F_{P_2} \in \Gamma_F$ and define
\[
H(x) := F_{P_1}(x)-F_{P_2}(x).
\]
We prove that for every $\omega\neq0$,
\[
\widehat H(\omega)
= \frac{\phi_{P_1}(-\omega)-\phi_{P_2}(-\omega)}{i\omega\sqrt{2\pi}},
\tag{\ref{eq:fourier_cdf_formula}}
\]
where $\phi_{P_i}(\omega)=\mathbb E_{X\sim P_i}[e^{i\omega X}]$ denotes the characteristic function of $P_i$.

By definition,
\[
\widehat H(\omega)
= \frac{1}{\sqrt{2\pi}}\int_{\mathbb R} H(x)e^{-i\omega x}\,dx,
\qquad \omega\neq0.
\]
We apply integration by parts in the sense of Stieltjes integrals. Let $u(x)=H(x)$ and $dv=e^{-i\omega x}dx$, so that $du=dH(x)$ and $v(x)=-e^{-i\omega x}/(i\omega)$. Then
\[
\int_{\mathbb R}H(x)e^{-i\omega x}\,dx
= \Bigl[H(x)v(x)\Bigr]_{-\infty}^{+\infty}
   -\int_{\mathbb R}v(x)\,dH(x).
\]
Since $H$ is of bounded variation and $H(x)\to0$ as $x\to\pm\infty$, while $v(x)=-e^{-i\omega x}/(i\omega)$ is bounded and continuous, the boundary term $\bigl[H(x)v(x)\bigr]_{-\infty}^{+\infty}$ vanishes in the sense of Stieltjes integration. Thus
\[
\int_{\mathbb R}H(x)e^{-i\omega x}\,dx
= \frac{1}{i\omega}\int_{\mathbb R} e^{-i\omega x}\,dH(x).
\]

The function $F_P$ is of bounded variation with Stieltjes measure $dF_P = P$, so $dH = dF_{P_1}-dF_{P_2} = P_1-P_2$. Therefore
\[
\int_{\mathbb R} e^{-i\omega x}\,dH(x)
= \int_{\mathbb R} e^{-i\omega x}\,P_1(dx)
  -\int_{\mathbb R} e^{-i\omega x}\,P_2(dx)
= \phi_{P_1}(-\omega)-\phi_{P_2}(-\omega).
\]
Combining with the previous display gives
\[
\int_{\mathbb R}H(x)e^{-i\omega x}\,dx
= \frac{\phi_{P_1}(-\omega)-\phi_{P_2}(-\omega)}{i\omega},
\]
and hence
\[
\widehat H(\omega)
= \frac{1}{\sqrt{2\pi}}\int_{\mathbb R}H(x)e^{-i\omega x}\,dx
= \frac{\phi_{P_1}(-\omega)-\phi_{P_2}(-\omega)}{i\omega\sqrt{2\pi}},
\]
which is exactly \eqref{eq:fourier_cdf_formula}.
\end{proof}

\subsection{Proof of Proposition~\ref{prop:cramer_spectral_bridge}}
\label{appendix:proof_cramer_convergence}

\begin{proof}
Let $P_1,P_2$ be probability measures with cumulative distribution functions $F_{P_1},F_{P_2}\in\Gamma_F$, and define
\[
\Delta\phi(\omega)
:= \phi_{P_1}(\omega)-\phi_{P_2}(\omega).
\]
By definition of the regularised spectral norm,
\begin{equation}
\|\Phi(P_1)-\Phi(P_2)\|_{\mathcal H_{\epsilon}}^2
= \frac{1}{2\pi}
  \int_{\mathbb R}
  \frac{|\Delta\phi(\omega)|^2}{\omega^2+\epsilon}\,d\omega.
\label{eq:A1}
\end{equation}

\textit{Monotone convergence.}
For every $\omega\neq0$, the function
\[
\epsilon \longmapsto \frac{1}{\omega^2+\epsilon}
\]
is monotonically decreasing as $\epsilon\downarrow0$, and satisfies
\[
\frac{1}{\omega^2+\epsilon}
\;\uparrow\;
\frac{1}{\omega^2}.
\]
Since $\phi_{P_i}(0)=1$ for $i=1,2$, we have $\Delta\phi(0)=0$, and hence the integrand in \eqref{eq:A1} is finite for all $\omega\in\mathbb R$. Therefore, for all $\omega\in\mathbb R$,
\[
\frac{|\Delta\phi(\omega)|^2}{\omega^2+\epsilon}
\;\uparrow\;
\frac{|\Delta\phi(\omega)|^2}{\omega^2}
\qquad (\epsilon\downarrow0).
\]

Moreover, the classical Cram\'er identity
\begin{equation}
d_C^2(P_1,P_2)
= \frac{1}{2\pi}
  \int_{\mathbb R}
  \frac{|\Delta\phi(\omega)|^2}{\omega^2}\,d\omega
<\infty
\label{eq:A5}
\end{equation}
implies that the limiting integrand is integrable. Since the integrand in \eqref{eq:A1} is nonnegative and increases pointwise as $\epsilon\downarrow0$, the monotone convergence theorem yields
\begin{equation}
\lim_{\epsilon\downarrow0}
\|\Phi(P_1)-\Phi(P_2)\|_{\mathcal H_{\epsilon}}^2
=
d_C^2(P_1,P_2).
\label{eq:A6}
\end{equation}
Taking square roots completes the proof.
\end{proof}

\subsection{Proof of Theorem~\ref{thm:cdf_spectral_isomorphism}}
\label{appendix:proof_cdf_spectral_isomorphism}

\begin{proof}
Fix $\epsilon>0$. We work throughout with the symmetric Fourier transform
\[
\widehat f(\omega)
 = \frac{1}{\sqrt{2\pi}}\int_{\mathbb R} f(x)\,e^{-i\omega x}\,dx,
\qquad
f(x)
 = \frac{1}{\sqrt{2\pi}}\int_{\mathbb R} \widehat f(\omega)\,e^{i\omega x}\,d\omega,
\]
and Plancherel's identity
\[
\int_{\mathbb R} |f(x)|^2\,dx
 = \int_{\mathbb R} |\widehat f(\omega)|^2\,d\omega .
\]

Recall that the centred CDF Hilbert space is
\[
\mathcal H^{\mathrm{cdf}}_\epsilon
 :=
 \Bigl\{
 H \mid
 \|H\|^2_{\mathcal H^{\mathrm{cdf}}_\epsilon}
 =
 \int_{\mathbb R}
 \frac{\omega^2}{\omega^2+\epsilon}
 |\widehat H(\omega)|^2\,d\omega
 < \infty
 \Bigr\},
\]
and the regularised spectral Hilbert space is
\[
\mathcal H_{\epsilon}
 :=
 \Bigl\{
 f \mid
 \|f\|^2_{\mathcal H_{\epsilon}}
 =
 \frac{1}{2\pi}
 \int_{\mathbb R}
 \frac{|\widehat f(\omega)|^2}{\omega^2+\epsilon}\,d\omega
 < \infty
 \Bigr\}.
\]

The operator
\[
\mathcal U :
\mathcal H^{\mathrm{cdf}}_\epsilon \longrightarrow \mathcal H_{\epsilon}
\]
is defined via its Fourier transform by
\[
\widehat{(\mathcal U H)}(\omega)
 :=
 -\sqrt{2\pi}\, i \omega\overline{\widehat H(\omega)},
 \qquad \omega\in\mathbb R .
\]

\textit{Conjugate-linearity.}
Let $\alpha,\beta\in\mathbb C$ and $H_1,H_2\in\mathcal H^{\mathrm{cdf}}_\epsilon$. By linearity of the Fourier transform and complex conjugation,
\begin{align*}
\widehat{(\mathcal U(\alpha H_1 + \beta H_2))}(\omega)
&=
-\sqrt{2\pi} i \omega[
\overline{\alpha \widehat H_1(\omega) + \beta \widehat H_2(\omega)} ]\\
&=
\overline{\alpha}\,\widehat{(\mathcal U H_1)}(\omega)
+
\overline{\beta}\,\widehat{(\mathcal U H_2)}(\omega),
\end{align*}
which shows that $\mathcal U$ is conjugate-linear.

\textit{Isometry.}
For $H\in\mathcal H^{\mathrm{cdf}}_\epsilon$,
\[
|\widehat{(\mathcal U H)}(\omega)|^2
= 2\pi\,\omega^2 |\widehat H(\omega)|^2.
\]
Substituting into the definition of the $\mathcal H_{\epsilon}$ norm yields
\begin{align*}
\|\mathcal U H\|^2_{\mathcal H_{\epsilon}}
&=
\frac{1}{2\pi}
\int_{\mathbb R}
\frac{|\widehat{(\mathcal U H)}(\omega)|^2}{\omega^2+\epsilon}\,d\omega \\
&=
\int_{\mathbb R}
\frac{\omega^2}{\omega^2+\epsilon}
|\widehat H(\omega)|^2\,d\omega
=
\|H\|^2_{\mathcal H^{\mathrm{cdf}}_\epsilon}.
\end{align*}
Thus $\mathcal U$ is an isometry from $\mathcal H^{\mathrm{cdf}}_\epsilon$ into $\mathcal H_{\epsilon}$.

\textit{Action on CDF differences.}
Let $P_1,P_2$ be probability measures with $F_{P_1},F_{P_2}\in\Gamma_F$, and set
\[
H := F_{P_1}-F_{P_2}.
\]
By Lemma~\ref{lemma:fourier_cdf_difference},
\[
\widehat H(\omega)
=
\frac{\phi_{P_1}(-\omega)-\phi_{P_2}(-\omega)}
     {i\omega\sqrt{2\pi}},
\qquad \omega\neq0.
\]
Substituting into the definition of $\mathcal U$ gives
\[
\widehat{(\mathcal U H)}(\omega)
=
\phi_{P_1}(\omega)-\phi_{P_2}(\omega).
\]
By definition of the spectral embedding \eqref{eq:phi_embedding_def}, this equals $\widehat{\Phi(P_1)}(\omega)-\widehat{\Phi(P_2)}(\omega)$. Hence
\[
\mathcal U(F_{P_1}-F_{P_2})
=
\Phi(P_1)-\Phi(P_2).
\]
In particular, taking $P_2=\delta_0$ yields
\[
\Phi(P)=\mathcal U(F_P-F_{\delta_0})
=\mathcal U(\mathcal U_C(P)).
\]

\textit{Invertibility.}
Define $\mathcal U^{-1} : \mathcal H_{\epsilon}\to\mathcal H^{\mathrm{cdf}}_\epsilon$ by
\[
\widehat{(\mathcal U^{-1} f)}(\omega)
 :=
 \frac{1}{\sqrt{2\pi} i \omega}\overline{\widehat f(\omega)}.
\]
A direct computation shows that $\mathcal U^{-1}\mathcal U=\mathrm{Id}$ on $\mathcal H^{\mathrm{cdf}}_\epsilon$ and $\mathcal U\mathcal U^{-1}=\mathrm{Id}$ on $\mathcal H_{\epsilon}$. Moreover,
\[
\|\mathcal U^{-1} f\|_{\mathcal H^{\mathrm{cdf}}_\epsilon}
=
\|f\|_{\mathcal H_{\epsilon}},
\]
so $\mathcal U^{-1}$ is bounded and isometric. This completes the proof.
\end{proof}

\subsection{Structural Properties of the CDF-Spectral Identification}
\label{appendix:cdf_spectral_structure}

This appendix provides a detailed structural analysis of the CDF-spectral identification introduced in Section~\ref{subsec:cdf_spectral_bridge} (Proposition~\ref{prop:canonical_transport}). In particular, we characterize the linear spans generated by centred CDF representations and their images under the regularised spectral embedding. These results are not required for the Bellman analysis in Section~\ref{sec:bellman_embedding}, but clarify the functional-analytic completeness of the transport construction.

\begin{proposition}[Raw-spectral correspondence on CDF embeddings]
\label{prop:raw_spectral_isomorphism}
Fix $\epsilon>0$ and work with the CDF class $\Gamma_F$. Define the set of centred CDF representations
\[
\Gamma_C
 := \Bigl\{
       \mathcal U_C(P) = F_P - F_{\delta_0} \;\mid\; F_P\in\Gamma_F
    \Bigr\}.
\]
Then $\Gamma_C \subset \mathcal H^{\mathrm{cdf}}_\epsilon$, and the linear span
\[
\mathrm{span}(\Gamma_C)
 := \Bigl\{
        \sum_{k=1}^n \alpha_k \mathcal U_C(P_k)
        \;\mid\;
        n\in\mathbb N,\ \alpha_k\in\mathbb R,\ F_{P_k}\in\Gamma_F
    \Bigr\}
\]
is a linear subspace of $\mathcal H^{\mathrm{cdf}}_\epsilon$. Moreover, the following statements hold:
\begin{enumerate}
\item For every $P$ with $F_P\in\Gamma_F$,
\[
\mathcal V(F_P)
  = \mathcal U\bigl(\mathcal U_C(P)\bigr)
  = \Phi(P).
\]

\item The restriction of $\mathcal U$ to $\mathrm{span}(\Gamma_C)$ is a linear isometric isomorphism onto the linear span of the corresponding spectral embeddings,
\[
\mathrm{span}\bigl\{\Phi(P) \mid F_P\in\Gamma_F\bigr\}
 \subset \mathcal H_{\epsilon},
\]
with inverse given by the restriction of $\mathcal U^{-1}$ to that span.

\item For any $P_1,P_2$ with $F_{P_1},F_{P_2}\in\Gamma_F$,
\[
\bigl\|\Phi(P_1) - \Phi(P_2)\bigr\|_{\mathcal H_{\epsilon}}
 = \bigl\|\mathcal U_C(P_1) - \mathcal U_C(P_2)\bigr\|_{\mathcal H^{\mathrm{cdf}}_\epsilon}.
\]
\end{enumerate}
\end{proposition}

\begin{proof}
By definition of $\Gamma_F$, if $F_P\in\Gamma_F$ then $\mathcal U_C(P) = F_P - F_{\delta_0} \in L^2(\mathbb R)$. Lemma~\ref{lemma:centered_cdf_membership} therefore implies that $\mathcal U_C(P)\in\mathcal H^{\mathrm{cdf}}_\epsilon$, so $\Gamma_C \subset \mathcal H^{\mathrm{cdf}}_\epsilon$ and $\mathrm{span}(\Gamma_C)$ is a linear subspace.

\medskip\noindent
\emph{Item~1.}
For any $P$ with $F_P\in\Gamma_F$, the definition of the transport operator $\mathcal V$ in~\eqref{eq:V_definition} gives
\[
\mathcal V(F_P)
 = \mathcal U(F_P - F_{\delta_0})
 = \mathcal U(\mathcal U_C(P)).
\]
By Theorem~\ref{thm:cdf_spectral_isomorphism}, $\mathcal U(\mathcal U_C(P))$ coincides with the regularised spectral embedding $\Phi(P)$, yielding the claimed identity.

\medskip\noindent
\emph{Item~2.}
By Theorem~\ref{thm:cdf_spectral_isomorphism}, $\mathcal U$ is a linear isometric isomorphism from $\mathcal H^{\mathrm{cdf}}_\epsilon$ onto $\mathcal H_{\epsilon}$. Its restriction to $\mathrm{span}(\Gamma_C)$ is therefore linear and norm-preserving.

To identify its image, let $H = \sum_{k=1}^n \alpha_k \mathcal U_C(P_k)$ with $F_{P_k}\in\Gamma_F$. By linearity of $\mathcal U$ and Item~1,
\[
\mathcal U(H)
 = \sum_{k=1}^n \alpha_k \mathcal U(\mathcal U_C(P_k))
 = \sum_{k=1}^n \alpha_k \Phi(P_k),
\]
which shows that
\[
\mathcal U\bigl(\mathrm{span}(\Gamma_C)\bigr)
 \subset
\mathrm{span}\bigl\{\Phi(P) : F_P\in\Gamma_F\bigr\}.
\]
Conversely, any finite linear combination $\sum_{k=1}^n \alpha_k \Phi(P_k)$ with $F_{P_k}\in\Gamma_F$ can be written as
\[
\sum_{k=1}^n \alpha_k \Phi(P_k)
 = \mathcal U\Bigl(\sum_{k=1}^n \alpha_k \mathcal U_C(P_k)\Bigr),
\]
establishing the reverse inclusion and hence equality of the two spans. Since $\mathcal U$ admits a bounded inverse $\mathcal U^{-1}$ on $\mathcal H_{\epsilon}$, its restriction yields the claimed inverse on the corresponding spans.

\medskip\noindent
\emph{Item~3.}
For any $P_1,P_2$ with $F_{P_1},F_{P_2}\in\Gamma_F$,
\[
\Phi(P_1) - \Phi(P_2)
 = \mathcal U(\mathcal U_C(P_1) - \mathcal U_C(P_2)),
\]
by linearity of $\mathcal U$. Since $\mathcal U$ is an isometry on $\mathcal H^{\mathrm{cdf}}_\epsilon$, taking norms yields
\[
\bigl\|\Phi(P_1) - \Phi(P_2)\bigr\|_{\mathcal H_{\epsilon}}
 = \bigl\|\mathcal U_C(P_1) - \mathcal U_C(P_2)\bigr\|_{\mathcal H^{\mathrm{cdf}}_\epsilon},
\]
as claimed.
\end{proof}

\begin{remark}[Geometric interpretation]
Proposition~\ref{prop:raw_spectral_isomorphism} shows that, after centring, raw CDF representations in $\Gamma_F$ generate an affine Hilbert structure that is carried isometrically into the linear geometry of the regularised spectral space $\mathcal H_{\epsilon}$. The operator $\mathcal V$ therefore provides a canonical bridge between probabilistic parameterisations and spectral Hilbert representations. This structural completeness result complements the canonical transport
properties used in the main text.
\end{remark}

\section{Proof Details in Section~\ref{sec:bellman_embedding}}

\subsection{Proof of Lemma~\ref{lemma:cdf_reward_nonexp}}
\label{appendix:proof_cdf_reward_nonexp}
\begin{proof}
Fix $(s,a)\in\mathcal S\times\mathcal A$. By definition of the reward translation operator,
\[
(\mathcal S_R F_{\mathcal Z_i})(s,a)(x)
=
\mathbb E_{r\sim R(s,a)}
\bigl[
F_{\mathcal Z_i(s,a)}(x-r)
\bigr],
\qquad i=1,2.
\]
Hence, for every $x\in\mathbb R$,
\[
(\mathcal S_R F_{\mathcal Z_1})(s,a)(x)
-
(\mathcal S_R F_{\mathcal Z_2})(s,a)(x)
=
\mathbb E_{r\sim R(s,a)}
\bigl[
F_{\mathcal Z_1(s,a)}(x-r)
-
F_{\mathcal Z_2(s,a)}(x-r)
\bigr].
\]
Applying Jensen's inequality pointwise in $x$ and integrating over $\mathbb R$
yields
\begin{align*}
\int_{\mathbb R}
\bigl|
(\mathcal S_R F_{\mathcal Z_1})(s,a)(x)
-
(\mathcal S_R F_{\mathcal Z_2})(s,a)(x)
\bigr|^2\,dx
&\le
\int_{\mathbb R}
\mathbb E_{r\sim R(s,a)}
\bigl[
|F_{\mathcal Z_1(s,a)}(x-r)-F_{\mathcal Z_2(s,a)}(x-r)|^2
\bigr]\,dx \\
&=
\mathbb E_{r\sim R(s,a)}
\left[
\int_{\mathbb R}
|F_{\mathcal Z_1(s,a)}(x-r)-F_{\mathcal Z_2(s,a)}(x-r)|^2\,dx
\right].
\end{align*}
For each fixed reward realisation $r$, the map $x\mapsto x-r$ is a translation of $\mathbb R$, and the $L^2(\mathbb R)$ norm is translation invariant. Therefore the inner integral equals
\[
\int_{\mathbb R}
|F_{\mathcal Z_1(s,a)}(x)-F_{\mathcal Z_2(s,a)}(x)|^2\,dx.
\]
Taking square roots and then the supremum over $(s,a)$ yields the claimed non-expansiveness in the metric $d_{\mathrm{Cr}}$.
\end{proof}

\subsection{Proof of Lemma~\ref{lemma:cdf_discount_contraction}}
\label{appendix:proof_cdf_discount_contraction}

\begin{proof}
Fix $(s,a)\in\mathcal S\times\mathcal A$. By definition of $\mathcal D_\gamma$,
\[
(\mathcal D_\gamma F_{\mathcal Z_i})(s,a)(x)
=
F_{\mathcal Z_i(s,a)}\!\left(\frac{x}{\gamma}\right),
\qquad i=1,2.
\]
Therefore,
\begin{align*}
&\int_{\mathbb R}
\bigl|
(\mathcal D_\gamma F_{\mathcal Z_1})(s,a)(x)
-
(\mathcal D_\gamma F_{\mathcal Z_2})(s,a)(x)
\bigr|^2\,dx \\
&\qquad =
\int_{\mathbb R}
\left|
F_{\mathcal Z_1(s,a)}\!\left(\frac{x}{\gamma}\right)
-
F_{\mathcal Z_2(s,a)}\!\left(\frac{x}{\gamma}\right)
\right|^2 dx.
\end{align*}
Performing the change of variables $u=x/\gamma$ (so $dx=\gamma\,du$) yields
\[
\int_{\mathbb R}
\bigl|
(\mathcal D_\gamma F_{\mathcal Z_1})(s,a)(x)
-
(\mathcal D_\gamma F_{\mathcal Z_2})(s,a)(x)
\bigr|^2\,dx
=
\gamma
\int_{\mathbb R}
|F_{\mathcal Z_1(s,a)}(u)-F_{\mathcal Z_2(s,a)}(u)|^2\,du.
\]
Taking square roots and then the supremum over $(s,a)$ gives
\[
d_{\mathrm{Cr}}(\mathcal D_\gamma F_{\mathcal Z_1},
                 \mathcal D_\gamma F_{\mathcal Z_2})
=
\sqrt{\gamma}\,
d_{\mathrm{Cr}}(F_{\mathcal Z_1},F_{\mathcal Z_2}),
\]
which proves the claim.
\end{proof}

\subsection{Proof of Lemma~\ref{lemma:cdf_condexp_nonexp}}
\label{appendix:proof_cdf_condexp_nonexp}
\begin{proof}
Fix $(s,a)\in\mathcal S\times\mathcal A$. By definition of $\mathcal C^\pi$, for any $x\in\mathbb R$ and $i\in\{1,2\}$,
\[
(\mathcal C^\pi F_{\mathcal Z_i})(s,a)(x)
=
\mathbb E_{(s',a')\sim \mathcal P^\pi(\cdot\mid s,a)}
\bigl[\,F_{\mathcal Z_i(s',a')}(x)\,\bigr].
\]
Since each $F_{\mathcal Z_i(s',a')}(\cdot)$ is a CDF and $\mathcal P^\pi(\cdot\mid s,a)$ is a probability measure, $(\mathcal C^\pi F_{\mathcal Z_i})(s,a)(\cdot)$ is a convex combination of CDFs and hence is itself a CDF. In particular, the Cram\'er distance at $(s,a)$ is well defined.

Define, for $(s',a')\in\mathcal S\times\mathcal A$ and $x\in\mathbb R$,
\[
\Delta(s',a';x)
:=
F_{\mathcal Z_1(s',a')}(x)-F_{\mathcal Z_2(s',a')}(x).
\]
Then for every $x\in\mathbb R$,
\begin{align*}
(\mathcal C^\pi F_{\mathcal Z_1})(s,a)(x)-(\mathcal C^\pi F_{\mathcal Z_2})(s,a)(x)
&=
\mathbb E_{(s',a')\sim \mathcal P^\pi(\cdot\mid s,a)}
\bigl[\Delta(s',a';x)\bigr].
\end{align*}
Applying Jensen's inequality to the square function, for each fixed $x\in\mathbb R$,
\[
\bigl|
\mathbb E_{(s',a')\sim \mathcal P^\pi(\cdot\mid s,a)}[\Delta(s',a';x)]
\bigr|^2
\le
\mathbb E_{(s',a')\sim \mathcal P^\pi(\cdot\mid s,a)}
\bigl[\,|\Delta(s',a';x)|^2\,\bigr].
\]
Hence, for every $x\in\mathbb R$,
\begin{equation}
\bigl|
(\mathcal C^\pi F_{\mathcal Z_1})(s,a)(x)-(\mathcal C^\pi F_{\mathcal Z_2})(s,a)(x)
\bigr|^2
\le
\mathbb E_{(s',a')\sim \mathcal P^\pi(\cdot\mid s,a)}
\bigl[\,|\Delta(s',a';x)|^2\,\bigr].
\label{eq:pointwise_jensen_condexp}
\end{equation}

Integrating \eqref{eq:pointwise_jensen_condexp} over $x$ and using monotonicity of the Lebesgue integral yields
\begin{align}
\int_{\mathbb R}
\bigl|
(\mathcal C^\pi F_{\mathcal Z_1})(s,a)(x)-(\mathcal C^\pi F_{\mathcal Z_2})(s,a)(x)
\bigr|^2\,dx
&\le
\int_{\mathbb R}
\mathbb E_{(s',a')\sim \mathcal P^\pi(\cdot\mid s,a)}
\bigl[\,|\Delta(s',a';x)|^2\,\bigr]\,dx.
\label{eq:integrate_pointwise}
\end{align}

To exchange the expectation and the integral on the right-hand side, note that the integrand
\[
|\Delta(s',a';x)|^2
=
|F_{\mathcal Z_1(s',a')}(x)-F_{\mathcal Z_2(s',a')}(x)|^2
\]
is nonnegative and measurable in $(s',a',x)$. Moreover, since $F_{\mathcal Z_1},F_{\mathcal Z_2}\in\mathcal X^{\mathrm{cdf}}$, we have
\[
\sup_{(s,a)\in\mathcal S\times\mathcal A}
\int_{\mathbb R}
|F_{\mathcal Z_1(s,a)}(x)-F_{\mathcal Z_2(s,a)}(x)|^2\,dx
=
d_{\mathrm{Cr}}^2(F_{\mathcal Z_1},F_{\mathcal Z_2})
<\infty.
\]
In particular, for the fixed $(s,a)$,
\begin{align*}
\mathbb E_{(s',a')\sim \mathcal P^\pi(\cdot\mid s,a)}
\Bigl[
\int_{\mathbb R}
|F_{\mathcal Z_1(s',a')}(x)-F_{\mathcal Z_2(s',a')}(x)|^2\,dx
\Bigr]
\le
d_{\mathrm{Cr}}^2(F_{\mathcal Z_1},F_{\mathcal Z_2})
<\infty,
\end{align*}
so $(s',a',x)\mapsto |\Delta(s',a';x)|^2$ is integrable with respect to the product measure $\mathcal P^\pi(\cdot\mid s,a)\otimes dx$. Therefore, by Fubini's theorem under integrability,
\begin{align}
\int_{\mathbb R}
\mathbb E_{(s',a')\sim \mathcal P^\pi(\cdot\mid s,a)}
\bigl[\,|\Delta(s',a';x)|^2\,\bigr]\,dx
&=
\int_{\mathbb R}
\int_{\mathcal S\times\mathcal A}
|\Delta(s',a';x)|^2\,\mathcal P^\pi(d(s',a')\mid s,a)\,dx \notag\\
&=
\int_{\mathcal S\times\mathcal A}
\int_{\mathbb R}
|\Delta(s',a';x)|^2\,dx\, \mathcal P^\pi(d(s',a')\mid s,a) \notag\\
&=
\mathbb E_{(s',a')\sim \mathcal P^\pi(\cdot\mid s,a)}
\Bigl[
\int_{\mathbb R}
|F_{\mathcal Z_1(s',a')}(x)-F_{\mathcal Z_2(s',a')}(x)|^2\,dx
\Bigr].
\label{eq:tonelli_exchange}
\end{align}

Combining \eqref{eq:integrate_pointwise} and \eqref{eq:tonelli_exchange} yields
\begin{align*}
\int_{\mathbb R}
\bigl|
(\mathcal C^\pi F_{\mathcal Z_1})(s,a)(x)-(\mathcal C^\pi F_{\mathcal Z_2})(s,a)(x)
\bigr|^2\,dx
&\le
\mathbb E_{(s',a')\sim \mathcal P^\pi(\cdot\mid s,a)}
\Bigl[
\int_{\mathbb R}
|F_{\mathcal Z_1(s',a')}(x)-F_{\mathcal Z_2(s',a')}(x)|^2\,dx
\Bigr] \\
&\le
\sup_{(s',a')\in\mathcal S\times\mathcal A}
\int_{\mathbb R}
|F_{\mathcal Z_1(s',a')}(x)-F_{\mathcal Z_2(s',a')}(x)|^2\,dx \\
&=
d_{\mathrm{Cr}}^2(F_{\mathcal Z_1},F_{\mathcal Z_2}).
\end{align*}

Taking square roots and then the supremum over $(s,a)$ gives
\[
d_{\mathrm{Cr}}\!\bigl(\mathcal C^\pi F_{\mathcal Z_1},\, \mathcal C^\pi F_{\mathcal Z_2}\bigr)
\le
d_{\mathrm{Cr}}\!\bigl(F_{\mathcal Z_1},\,F_{\mathcal Z_2}\bigr),
\]
as claimed.
\end{proof}

\subsection{Proof of Lemma~\ref{lem:P2_complete}}
\label{appendix:proof_cramer_cdf_complete}
\begin{proof}
Let $\{F_n\}_{n\ge 0}$ be a Cauchy sequence in $(\Gamma_F,d_C)$.
By definition of $d_C$, this means
\[
\|F_n - F_m\|_{L^2(\mathbb R)} \longrightarrow 0
\qquad\text{as } n,m\to\infty.
\]
Hence $\{F_n\}_{n\ge 0}$ is Cauchy in $L^2(\mathbb R)$ and there exists
$G\in L^2(\mathbb R)$ such that
\[
\|F_n - G\|_{L^2(\mathbb R)} \longrightarrow 0.
\]
Passing to a subsequence if necessary, we may assume that $F_n(x) \to G(x)$ for almost every $x \in \mathbb R$. Each $F_n$ is a cumulative distribution function: it is nondecreasing, right-continuous, takes values in $[0,1]$, and satisfies $\lim_{x\to-\infty} F_n(x)=0$ and $\lim_{x\to+\infty} F_n(x)=1$. From the almost-everywhere convergence and the uniform bound $0 \le F_n \le 1$, it follows that $0 \le G \le 1$ almost everywhere. Moreover, since $F_n \to G$ almost everywhere and each $F_n$ is nondecreasing, there exists a version of $G$, equal almost everywhere, that is nondecreasing on $\mathbb R$.

We now regularize $G$ to a right-continuous version.
Define
\[
\widetilde F(x)
 := \lim_{y\downarrow x} G(y), \qquad x\in\mathbb R.
\]
Since $G$ is nondecreasing, the limit on the right exists for every $x$. Then $\widetilde F$ is nondecreasing and right--continuous, and satisfies $0\le\widetilde F\le 1$ everywhere. Furthermore, $\widetilde F = G$ almost everywhere, so
\[
\|F_n - \widetilde F\|_{L^2(\mathbb R)}
 = \|F_n - G\|_{L^2(\mathbb R)}
 \longrightarrow 0
\qquad\text{as } n\to\infty.
\]

It remains to show that $\widetilde F$ has the correct boundary behaviour at $\pm\infty$ and that its Cram\'er distance to $\delta_0$ is finite. Recall that $F_{\delta_0}(x)=0$ for $x<0$ and $F_{\delta_0}(x)=1$ for $x\ge0$. Since $F_n\in\Gamma_F$, we have
\[
\int_{\mathbb R} \bigl(F_n(x) - F_{\delta_0}(x)\bigr)^2 dx
 = \ell_2^2(F_n,F_{\delta_0}) < \infty
\qquad\text{for every }n.
\]
Because $F_n\to\widetilde F$ in $L^2(\mathbb R)$, we obtain
\[
\int_{\mathbb R}
  \bigl(\widetilde F(x) - F_{\delta_0}(x)\bigr)^2 dx
 = \lim_{n\to\infty}
    \int_{\mathbb R} \bigl(F_n(x) - F_{\delta_0}(x)\bigr)^2 dx
 < \infty.
\]
In particular,
\[
\int_{-\infty}^0 \widetilde F(x)^2 dx
  + \int_0^{\infty} \bigl(1-\widetilde F(x)\bigr)^2 dx
 < \infty.
\]

We now deduce the boundary limits. Suppose first that $\liminf_{x\to-\infty} \widetilde F(x) = \eta > 0$. Then there exists $M>0$ such that $\widetilde F(x)\ge \eta/2$ for all $x<-M$. Since the interval $(-\infty,-M)$ has infinite Lebesgue measure, this would
imply
\[
\int_{-\infty}^0 \widetilde F(x)^2 dx
 \ge \int_{-\infty}^{-M} \Bigl(\frac{\eta}{2}\Bigr)^2 dx
 = \infty,
\]
contradicting finiteness of the integral. Hence
\[
\lim_{x\to-\infty} \widetilde F(x) = 0.
\]

A similar argument applies at $+\infty$. If $\limsup_{x\to+\infty} \widetilde F(x) \le 1-\eta$ for some $\eta>0$, then there exists $M>0$ such that $\widetilde F(x)\le 1-\eta/2$ for all $x>M$, and therefore
\[
\int_0^{\infty} \bigl(1-\widetilde F(x)\bigr)^2 dx
 \ge \int_M^{\infty} \Bigl(\frac{\eta}{2}\Bigr)^2 dx
 = \infty,
\]
again a contradiction. Thus
\[
\lim_{x\to+\infty} \widetilde F(x) = 1.
\]

We have shown that $\widetilde F$ is nondecreasing, right-continuous, takes values in $[0,1]$, satisfies the boundary conditions
$\widetilde F(-\infty)=0$ and $\widetilde F(+\infty)=1$, and obeys
\[
\int_{\mathbb R}
  \bigl(\widetilde F(x) - F_{\delta_0}(x)\bigr)^2 dx < \infty.
\]
Hence $\widetilde F$ is the cumulative distribution function of a probability measure $P$ with $\ell_2(F_P,F_{\delta_0}) < \infty$, that is, $\widetilde F\in\Gamma_F$.

Finally, together with $ \|F_n - \widetilde F\|_{L^2(\mathbb R)}  \rightarrow 0,$ thus $\{F_n\}_{n\ge 0}$ converges to $\widetilde F$ in $(\Gamma_F,d_C)$. This proves that $(\Gamma_F,d_C)$ is complete.
\end{proof}

\subsection{Proof of Corollary~\ref{cor:cdf_fixed_point}}
\label{appendix:proof_cdf_fixed_point}
\begin{proof}
By Theorem~\ref{thm:cramer_contraction}, $\mathcal T^\pi_{\mathrm{cdf}}$ is a strict contraction with constant $\sqrt{\gamma}<1$ on $(\mathcal X^{\mathrm{cdf}},d_{\mathrm{Cr}})$.

We first show that $(\mathcal X^{\mathrm{cdf}},d_{\mathrm{Cr}})$ is complete. Recall that for any two CDF fields
$F_{\mathcal Z_1},F_{\mathcal Z_2}\in\mathcal X^{\mathrm{cdf}}$,
\[
d_{\mathrm{Cr}}(F_{\mathcal Z_1},F_{\mathcal Z_2})
 = \sup_{(s,a)\in\mathcal S\times\mathcal A}
    d_C\bigl(
      F_{\mathcal Z_1(s,a)},
      F_{\mathcal Z_2(s,a)}
    \bigr),
\]
where $d_C$ denotes the Cram\'er metric on the admissible CDF class $\Gamma_F$. By Lemma~\ref{lem:P2_complete}, $(\Gamma_F,d_C)$ is complete.

Let $\{F_{\mathcal Z_n}\}_{n\ge 0}$ be a $d_{\mathrm{Cr}}$-Cauchy sequence in $\mathcal X^{\mathrm{cdf}}$. Then for every $(s,a)$, the sequence $\{F_{\mathcal Z_n(s,a)}\}_{n\ge 0}$ is Cauchy in $(\Gamma_F,d_C)$ and hence converges to some limit $F_{\mathcal Z(s,a)}\in\Gamma_F$. This defines a pointwise limit CDF field
\[
F_{\mathcal Z} :
\mathcal S\times\mathcal A \longrightarrow \Gamma_F.
\]

We next verify that $F_{\mathcal Z}\in\mathcal X^{\mathrm{cdf}}$ and that $F_{\mathcal Z_n}\to F_{\mathcal Z}$ in $d_{\mathrm{Cr}}$. Fix $(s_0,a_0)\in\mathcal S\times\mathcal A$ and an index $n_0$. Since $F_{\mathcal Z_{n_0}}\in\mathcal X^{\mathrm{cdf}}$, there exists $M<\infty$ such that
\[
\sup_{(s,a)}
d_C\bigl(
F_{\mathcal Z_{n_0}(s,a)},
F_{\mathcal Z_{n_0}(s_0,a_0)}
\bigr)
\le M.
\]
Moreover, since $\{F_{\mathcal Z_n}\}_{n\ge 0}$ is Cauchy in $d_{\mathrm{Cr}}$, there exists $K<\infty$ such that
\[
\sup_{n\ge n_0}
d_{\mathrm{Cr}}\bigl(F_{\mathcal Z_n},F_{\mathcal Z_{n_0}}\bigr)
\le K.
\]
For any $(s,a)$ and $n\ge n_0$,
\[
d_C\bigl(
F_{\mathcal Z_n(s,a)},
F_{\mathcal Z_{n_0}(s_0,a_0)}
\bigr)
\le
d_C\bigl(
F_{\mathcal Z_n(s,a)},
F_{\mathcal Z_{n_0}(s,a)}
\bigr)
+
d_C\bigl(
F_{\mathcal Z_{n_0}(s,a)},
F_{\mathcal Z_{n_0}(s_0,a_0)}
\bigr)
\le K+M.
\]
Passing to the limit $n\to\infty$ yields
\[
d_C\bigl(
F_{\mathcal Z(s,a)},
F_{\mathcal Z_{n_0}(s_0,a_0)}
\bigr)
\le K+M,
\]
so $F_{\mathcal Z}$ is a bounded CDF field and hence
$F_{\mathcal Z}\in\mathcal X^{\mathrm{cdf}}$.

Finally, to show convergence in $d_{\mathrm{Cr}}$, fix $\epsilon>0$ and choose $N$ such that $d_{\mathrm{Cr}}(F_{\mathcal Z_n},F_{\mathcal Z_m})<\epsilon$ for all $n,m\ge N$. Then for any fixed $n\ge N$ and any $(s,a)$,
\[
d_C\bigl(
F_{\mathcal Z_n(s,a)},
F_{\mathcal Z_m(s,a)}
\bigr)
\le
d_{\mathrm{Cr}}(F_{\mathcal Z_n},F_{\mathcal Z_m})
<\epsilon
\quad\text{for all } m\ge N.
\]
Letting $m\to\infty$ and using pointwise convergence yields
\[
d_C\bigl(
F_{\mathcal Z_n(s,a)},
F_{\mathcal Z(s,a)}
\bigr)
\le \epsilon
\quad\text{for all } (s,a).
\]
Taking the supremum over $(s,a)$ gives
\[
d_{\mathrm{Cr}}(F_{\mathcal Z_n},F_{\mathcal Z})
\le \epsilon,
\]
so $F_{\mathcal Z_n}\to F_{\mathcal Z}$ in $(\mathcal X^{\mathrm{cdf}},d_{\mathrm{Cr}})$. Thus $(\mathcal X^{\mathrm{cdf}},d_{\mathrm{Cr}})$ is complete.

The Banach fixed-point theorem then yields a unique fixed point $F_{\mathcal Z^\pi}\in\mathcal X^{\mathrm{cdf}}$ of $\mathcal T^\pi_{\mathrm{cdf}}$, and convergence of the iterates
\[
F_{\mathcal Z_{n+1}}=\mathcal T^\pi_{\mathrm{cdf}}F_{\mathcal Z_n}
\]
to $F_{\mathcal Z^\pi}$ in $d_{\mathrm{Cr}}$ for any initial $F_{\mathcal Z_0}\in\mathcal X^{\mathrm{cdf}}$.
\end{proof}

\vskip 0.2in
\bibliography{sample}

\end{document}